\documentclass[11pt]{article}

\usepackage{dsfont}

\def\colorful{0}

\usepackage{amsthm,amsfonts,amsmath,amssymb,epsfig,color,float,graphicx,verbatim, enumitem}
\usepackage{multirow}
\usepackage{algorithm}
\usepackage[noend]{algpseudocode}

\newif\ifhyper\IfFileExists{hyperref.sty}{\hypertrue}{\hyperfalse}
\hypertrue
\ifhyper\usepackage{hyperref}\fi

\usepackage{enumitem}

\usepackage{framed}
\usepackage{nicefrac}

\def\nnewcolor{1}
\ifnum\nnewcolor=1

\fi
\ifnum\nnewcolor=0

\fi

\ifnum\colorful=1
\newcommand{\new}[1]{{\color{red} #1}}
\newcommand{\green}[1]{{\color{green} #1}}
\newcommand{\yellow}[1]{{\color{yellow} #1}}
\newcommand{\newblue}[1]{{\color{blue} #1}}
\else
\newcommand{\new}[1]{{#1}}

\newcommand{\newblue}[1]{{#1}}
\newcommand{\yellow}[1]{{#1}}
\newcommand{\green}[1]{{#1}}
\fi

\newtheorem{theorem}{Theorem}[section]

\newtheorem{lemma}[theorem]{Lemma}
\newtheorem{informal theorem}[theorem]{Theorem (informal statement)}

\newtheorem{claim}[theorem]{Claim}

\newtheorem{remark}[theorem]{Remark}

\theoremstyle{definition}
\newtheorem{definition}[theorem]{Definition}
\newcommand{\eqdef}{\stackrel{{\mathrm {\footnotesize def}}}{=}}

\newcommand{\bx}{\mathbf{x}}
\newcommand{\by}{\mathbf{y}}
\newcommand{\bv}{\mathbf{v}}
\newcommand{\bw}{\mathbf{w}}
\newcommand{\br}{\mathbf{r}}

\newcommand{\e}{\mathbf{e}}

\newcommand{\err}{\mathrm{err}}
\newcommand{\LTF}{\mathcal{H}}
\newcommand{\B}{\mathbb{B}}

\newcommand{\R}{\mathbb{R}}

\newcommand{\Z}{\mathbb{Z}}

\newcommand{\E}{\mathbf{E}}
\newcommand{\eps}{\epsilon}

\newcommand{\pr}{\mathbf{Pr}}
\renewcommand{\Pr}{\mathbf{Pr}}
\newcommand{\poly}{\mathrm{poly}}

\newcommand{\sgn}{\mathrm{sign}}
\newcommand{\sign}{\mathrm{sign}}

\newcommand{\Ex}{\mathop{{\bf E}\/}}
\newcommand{\opt}{\mathrm{OPT}}
\newcommand{\D}{\mathcal{D}}
\newcommand{\LR}{\mathrm{LeakyRelu}}
\newcommand{\Ind}{\mathds{1}}

\newcommand{\littlesum}{\mathop{\textstyle \sum}}

\newcommand{\be}{\mathbf{e}}

\newcommand{\wh}{\widehat}

\newcommand{\ltwo}[1]{\left\lVert#1\right\rVert_2}
\newcommand{\dotp}[2]{\langle #1, #2 \rangle}

\title{Distribution-Independent PAC Learning of Halfspaces\\ with Massart Noise}

\author{
Ilias Diakonikolas\thanks{Supported by NSF Award CCF-1652862 (CAREER) and a Sloan Research Fellowship. Part of this 
work was performed at the Simons Institute for the Theory of Computing during the program on Foundations of Data Science.}\\
University of Wisconsin-Madison\\
{\tt ilias@cs.wisc.edu}\\
\and
Themis Gouleakis\thanks{This research was performed while the author was a postdoctoral researcher at USC,
hosted by Ilias Diakonikolas.}\\
Max Planck Institute for Informatics\\
{\tt themis.gouleakis@gmail.com}\\
\and
Christos Tzamos\\ University of Wisconsin-Madison\\
{\tt tzamos@wisc.edu}
}

\begin{document}

\maketitle

\begin{abstract}
We study the problem of {\em distribution-independent} PAC learning of halfspaces in the presence of Massart noise. 
Specifically, we are given a set of labeled examples $(\bx, y)$ drawn 
from a distribution $\D$ on $\R^{d+1}$ such that the marginal distribution 
on the unlabeled points $\bx$ is arbitrary and the labels $y$ are generated by an unknown halfspace 
corrupted with Massart noise at noise rate $\eta<1/2$. The goal is to find 
a hypothesis $h$ that minimizes the misclassification error $\pr_{(\bx, y) \sim \D} \left[ h(\bx) \neq y \right]$. 

We give a $\poly\left(d, 1/\eps\right)$ time algorithm for this problem with misclassification error $\eta+\eps$. 
We also provide evidence that improving on the error guarantee of our algorithm
might be computationally hard. Prior to our work, no efficient weak (distribution-independent) learner 
was known in this model, even for the class of disjunctions. The existence of such an algorithm 
for halfspaces (or even disjunctions) has been posed as an open question in various works, 
starting with Sloan (1988), Cohen (1997), and was most recently highlighted in Avrim Blum's FOCS 2003 tutorial.
\end{abstract}

\setcounter{page}{0}

\thispagestyle{empty}

\newpage

\section{Introduction} \label{sec:intro}


Halfspaces, or Linear Threshold Functions (henceforth LTFs), 
are Boolean functions $f: \R^d \to \{ \pm 1\}$ of the form 
$f(\bx) = \sgn(\langle \bw, \bx \rangle- \theta)$, where $\bw \in \R^d$ is the weight vector and $\theta \in \R$ is the threshold. 
(The function $\sign: \R \to \{ \pm 1\}$ is defined as $\sgn(u)=1$ if $u \geq 0$ and $\sgn(u)=-1$ otherwise.)
The problem of learning an unknown halfspace is as old as the field of machine learning --- starting 
with Rosenblatt's Perceptron algorithm~\cite{Rosenblatt:58} --- and has arguably been the most influential problem 
in the development of the field. In the realizable setting, LTFs are known to be efficiently learnable 
in Valiant's distribution-independent PAC model~\cite{val84} via Linear Programming~\cite{MT:94}. 
In the presence of corrupted data, the situation is more subtle and crucially depends on the underlying noise model. 
In the agnostic model~\cite{Haussler:92, KSS:94} -- where an adversary is allowed to arbitrarily corrupt 
an arbitrary $\eta<1/2$ fraction of the labels -- even weak learning is known to be computationally 
intractable~\cite{GR:06, FGK+:06short, Daniely16}. On the other hand, 
in the presence of Random Classification Noise (RCN)~\cite{AL88} -- where each label is flipped
independently with probability exactly $\eta<1/2$ -- a polynomial time algorithm is known~\cite{BlumFKV96, BFK+:97}. 

\noindent In this work, we focus on learning halfspaces with Massart noise~\cite{Massart2006}:
\begin{definition}[Massart Noise Model] \label{def:massart-learning}
Let $\mathcal{C}$ be a class of Boolean functions over $X= \R^d$, $\D_{\bx}$ 
be an arbitrary distribution over $X$, and $0 \leq \eta <1/2$. 
Let $f$ be an unknown target function in $\mathcal{C}$. 
A {\em noisy example oracle}, $\mathrm{EX}^{\mathrm{Mas}}(f, \D_{\bx}, \eta)$,
works as follows: Each time $\mathrm{EX}^{\mathrm{Mas}}(f, \D_{\bx}, \eta)$ is invoked,
it returns a labeled example $(\bx, y)$, where $\bx \sim \D_{\bx}$, $y = f(\bx)$ with 
probability $1-\eta(\bx)$ and $y = -f(\bx)$ with probability $\eta(\bx)$, 
for an {\em unknown} parameter $\eta(\bx) \leq \eta$.
Let $\D$ denote the joint distribution on $(\bx, y)$ generated by the above oracle.
A learning algorithm is given i.i.d. samples from $\D$
and its goal is to output a hypothesis $h$ such that with high probability
the error $\pr_{(\bx, y) \sim \D} [h(\bx) \neq y]$ is small.
\end{definition}

\noindent An equivalent formulation of the Massart model~\cite{Sloan88, Sloan92} is the following: 
With probability $1-\eta$, we have that $y = f(\bx)$, and with probability $\eta$
the label $y$ is controlled by an adversary. Hence, the Massart model lies in between the RCN and the agnostic models.
(Note that the RCN model corresponds to the special case that 
$\eta(\bx) = \eta$ for all $\bx \in X$.)
It is well-known (see, e.g.,~\cite{Massart2006}) that \newblue{$\poly(d, 1/\eps)$} samples 
information-theoretically suffice to compute a hypothesis with misclassification error 
$\opt+ \eps$, where $\opt$ is the misclassification error of the optimal halfspace. 
Also note that $\opt \leq \eta$ by definition. The question is whether a polynomial
time algorithm exists.

The existence of an efficient distribution-independent learning algorithm for halfspaces
(or even disjunctions) in the Massart model has been posed as an open question 
in a number of works. In the first COLT conference~\cite{Sloan88} (see also~\cite{Sloan92}), 
Sloan defined the malicious misclassification noise model 
(an equivalent formulation of Massart noise, described above)
and asked  whether there exists an efficient learning algorithm for disjunctions 
in this model. About a decade later, Cohen~\cite{Cohen:97} 
asked the same question for the more general class of all LTFs.
The question remained open --- even for weak learning of disjunctions! --- and was highlighted 
in Avrim Blum's FOCS 2003 tutorial~\cite{Blum03}. Specifically, prior to this work, 
even the following very basic special case remained open:
\begin{center}
{\em Given labeled examples from an unknown disjunction, corrupted with $1\%$ Massart noise,\\
can we efficiently find a hypothesis that achieves misclassification error $49\%$?}
\end{center}
The reader is referred to slides 39-40 of Avrim Blum's FOCS'03 tutorial~\cite{Blum03},
where it is suggested that the above problem might be easier than agnostically learning
disjunctions. As a corollary of our main result (Theorem~\ref{thm:main-informal}), we
answer this question in the affirmative. In particular, we obtain an efficient algorithm
that achieves misclassification error arbitrarily close to $\eta$ for all LTFs.

\subsection{Our Results} \label{ssec:results}

The main result of this paper is the following:

\begin{theorem}[Main Result] \label{thm:main-informal}
There is an algorithm that for all $0 < \eta < 1/2$, on input  
a set of i.i.d. examples from a distribution 
$\D= \mathrm{EX}^{\mathrm{Mas}}(f, \D_{\bx}, \eta)$ on $\R^{d+1}$, 
where $f$ is an unknown halfspace on $\R^d$,
it runs in $\poly(d, b, 1/\eps)$ time, where $b$ is an upper bound 
on the bit complexity of the examples, and outputs a hypothesis $h$ 
that with high probability satisfies $\pr_{(\bx, y) \sim \D} [h(\bx) \neq y] \leq \eta +\eps$.
\end{theorem}

\new{
\noindent See Theorem~\ref{thm:general-case} for a more detailed formal statement.
For large-margin halfspaces, we obtain a slightly better error guarantee; 
see Theorem~\ref{thm:margin-case} and Remark~\ref{rem:margin-error}.}

\medskip
\noindent {\bf Discussion.}
We note that our algorithm is non-proper, i.e., the hypothesis $h$
itself is not a halfspace. The polynomial dependence on $b$ in the runtime cannot be removed, 
even in the noiseless case, unless one obtains strongly-polynomial algorithms 
for linear programming. Finally, we note that the misclassification error of 
$\eta+\eps$ translates to error \new{$2\eta+\eps$}
with respect to the target LTF.

\noindent Our algorithm gives error $\eta+\eps$, instead of the information-theoretic optimum 
of $\opt+\eps$. To complement our positive result,
we provide some evidence that improving on our $(\eta+\eps)$ error guarantee 
may be challenging. Roughly speaking, we show 
(see Theorems~\ref{thm:lb-surr} and~\ref{thm:lb-thresh}) 
that natural approaches --- involving convex surrogates and refinements thereof ---
inherently fail, \new{even under margin assumptions}. (See Section~\ref{ssec:techniques} for a discussion.)

\medskip
\noindent {\bf Broader Context.}
This work is part of the broader agenda of designing robust estimators
in the distribution-independent setting with respect to natural noise models. 
A recent line of work~\cite{KLS09, ABL17, DKKLMS16, LaiRV16, DKK+17, DKKLMS18-soda, 
DKS18a, KlivansKM18, DKS19, DKK+19-sever}
has given efficient robust estimators for a range of learning tasks (both supervised and unsupervised)
in the presence of a small constant fraction of adversarial corruptions. 
A limitation of these results is the assumption that the good data 
comes from a ``tame'' distribution, e.g., Gaussian or isotropic log-concave distribution. 
On the other hand, if {\em no} assumption is made on the good data 
and the noise remains fully adversarial, these problems
become computationally intractable~\cite{Bernholt, GR:06, Daniely16}. This suggests 
the following general question:
{\em Are there realistic noise models that allow for efficient algorithms 
without imposing (strong) assumptions on the good data?}
Conceptually, the algorithmic results of this paper could be viewed as an affirmative 
answer to this question for the problem of learning halfspaces.

\subsection{Technical Overview} \label{ssec:techniques}
In this section, we provide an outline of our approach 
and a comparison to previous techniques.
Since the distribution on the unlabeled data is arbitrary, we can assume w.l.o.g.
that the threshold $\theta=0$.

\paragraph{Massart Noise versus RCN.}
Random Classification Noise (RCN)~\cite{AL88} is the special case of Massart
noise where each label is flipped with probability {\em exactly} $\eta <1/2$.
At first glance, it might seem that Massart noise is easier to deal with computationally than RCN.
After all, in the Massart model we add {\em at most as much noise} as in the RCN model 
\new{with noise rate $\eta$}. 
It turns out that this intuition is fundamentally flawed. Roughly speaking, 
the ability of the Massart adversary to choose {\em whether} to perturb a given label 
and, if so, with what probability (which is {\em unknown} to the learner), 
makes the design of efficient algorithms in this model challenging. 
In particular, the well-known connection between learning with RCN
and the Statistical Query (SQ) model~\cite{Kearns93, Kearns:98} 
no longer holds, i.e., the property of being an SQ algorithm does {\em not} automatically suffice 
for noise-tolerant learning with Massart noise. We note that this connection with the SQ model
is leveraged in~\cite{BlumFKV96, BFK+:97} to obtain their polynomial time
algorithm for learning halfspaces with RCN.

\paragraph{Large Margin Halfspaces.}
To illustrate our approach, we start by describing
our learning algorithm for {\em $\gamma$-margin} halfspaces on the unit ball.
That is, we assume $| \langle \bw^{\ast}, \bx \rangle | \geq \gamma$ for every $\bx$ in the support, 
where $\bw^{\ast} \in \R^d$ with $\| \bw^{\ast} \|_2 = 1$ defines the target halfspace 
$h_{\bw^{\ast}}(\bx)  = \sgn(\langle \bw^{\ast}, \bx \rangle)$. Our goal is to design
a $\poly(d, 1/\eps, 1/\gamma)$ time learning algorithm in the presence of Massart noise.

In the RCN model, the large margin case is easy 
because the learning problem is essentially convex. That is, there is a convex surrogate
that allows us to formulate the problem 
as a convex program. We can use SGD to find a near-optimal solution to this convex program, 
which automatically gives a {\em strong proper} learner. This simple fact
does not appear explicitly in the literature, but follows easily from standard tools.
\cite{Bylander94} showed that a variant of the Perceptron algorithm
(which can be viewed as gradient descent on a particular convex objective)
learns $\gamma$-margin halfspaces in $\poly(d, 1/\eps, 1/\gamma)$ time. 
The algorithm in~\cite{Bylander94} requires an additional \new{anti-concentration} 
condition about the distribution, which is easy to remove. 
In Appendix~\ref{sec:rcn}, we show that a ``smoothed'' version of Bylander's 
objective suffices as a convex surrogate under only the margin assumption.

Roughly speaking, the reason that a convex surrogate works for RCN 
is that the expected effect of the noise on each label is known a priori. 
Unfortunately, this is not the case for Massart noise.
We show (Theorem~\ref{thm:lb-surr} in Appendix~\ref{sec:lb}) 
that no convex surrogate can lead to a {\em weak learner},
even under a margin assumption.
That is, if $\wh{\bw}$ is the minimizer of $G(\bw)=\E_{(\bx,y)\sim \D}[\phi(y\langle \bw, \bx\rangle)]$,
where $\phi$ can be any convex function, then the hypothesis $\sgn(\langle \wh{\bw}, \bx\rangle)$
is not even a weak learner. So, in sharp contrast with the RCN case, the problem is non-convex 
in this sense.

Our Massart learning algorithm for large margin halfspaces still uses
a convex surrogate, but in a qualitatively different way. 
Instead of attempting to solve the problem in one-shot, 
our algorithm adaptively applies a sequence of convex optimization problems 
to obtain an accurate solution in disjoint subsets of the space. 
Our iterative approach is motivated by a new structural lemma (Lemma~\ref{lm:structural}) 
establishing the following: {\em Even though minimizing a convex proxy does not lead to small 
misclassification error over the entire space, there exists a region with non-trivial probability
mass where it does.} Moreover, this region is efficiently identifiable by a simple thresholding rule.
Specifically, we show that there exists a threshold $T>0$ (which can be found algorithmically) 
such that the hypothesis $\sgn(\langle \wh{\bw}, \bx\rangle)$ has error bounded by $\eta+\eps$
in the region $R_T = \{\bx: |\langle \wh{\bw}, \bx \rangle |  \geq T \}$. Here $\wh{\bw}$ is any
near-optimal solution to an appropriate convex optimization problem, defined 
via a convex surrogate objective similar to the one used in~\cite{Bylander94}. 
We note that Lemma~\ref{lm:structural}
is the main technical novelty of this paper and motivates our algorithm.
Given Lemma~\ref{lm:structural}, in any iteration $i$
we can find the best threshold $T^{(i)}$ using samples, 
and obtain a learner with misclassification error $\eta+\eps$ in the corresponding region. 
Since each region has non-trivial mass, iterating this scheme a small number of times 
allows us to find a non-proper hypothesis (a decision-list of halfspaces) 
with misclassification error at most $\eta+\eps$ in the entire space.

The idea of iteratively optimizing a convex surrogate was used in~\cite{BlumFKV96} 
to learn halfspaces with RCN {\em without} a margin. 
Despite this similarity, we note that the algorithm of ~\cite{BlumFKV96} fails 
to even obtain a weak learner in the Massart model.
We point out two crucial technical differences: First, the iterative approach in~\cite{BlumFKV96} 
was needed to achieve polynomial running time. As mentioned already,
a convex proxy is guaranteed to converge to the true solution with RCN, 
but the convergence may be too slow (when the margin is tiny). 
In contrast, with Massart noise (even under a margin condition) 
convex surrogates cannot even give weak learning in the entire domain. 
Second, the algorithm of~\cite{BlumFKV96} used a fixed threshold in each iteration, 
equal to the margin parameter obtained after an appropriate 
pre-processing of the data (that is needed in order to ensure a weak margin property). 
In contrast, in our setting, we need to find 
an appropriate threshold $T^{(i)}$ in each iteration $i$, 
according to the criterion specified by our Lemma~\ref{lm:structural}. 

\paragraph{General Case.}
Our algorithm for the general case (in the absence of a margin)
is qualitatively similar to our algorithm for the large margin case, but the details are more elaborate.
We borrow an idea from~\cite{BlumFKV96} that in some sense allows
us to ``reduce'' the general case to the large margin case. Specifically,
~\cite{BlumFKV96} (see also~\cite{DV:04}) developed 
a pre-processing routine that slightly modifies the distribution 
on the unlabeled points and guarantees the following {\em weak margin} property:
After pre-processing, there exists an explicit margin parameter $\sigma = \Omega(1/\poly(d, b))$, 
such that any hyperplane through the origin has at least a non-trivial
mass of the distribution at distance at least $\sigma$ from it. 
Using this pre-processing step, we are able to adapt our algorithm 
from the previous subsection to work without margin assumptions in 
$\poly(d, b, 1/\eps)$ time. While our analysis is similar in spirit
to the case of large margin, we note that the margin property
obtained via the~\cite{BlumFKV96,DV:04} preprocessing step 
is (necessarily) weaker, hence additional careful analysis is required.


\paragraph{Lower Bounds Against Natural Approaches.}
We have already explained our Theorem~\ref{thm:lb-surr}, which shows
that using a convex surrogate over the entire space cannot not give a weak learner.
Our algorithm, however, can achieve error $\eta+\eps$ by iteratively optimizing
a specific convex surrogate in disjoint subsets of the domain. 
A natural question is whether one can obtain qualitatively better accuracy, e.g., 
$f(\opt)+\eps$, by using a {\em different} convex objective function in our iterative thresholding approach. 
We show (Theorem~\ref{thm:lb-thresh}) that such an improvement is not possible: Using
a different convex proxy cannot lead to error better than $(1-o(1)) \cdot \eta$.
It is a plausible conjecture that improving on the error guarantee of our algorithm is computationally hard.
We leave this as an intriguing open problem for future work.

\subsection{Prior and Related Work} \label{ssec:related}
Bylander~\cite{Bylander94} gave a polynomial time 
algorithm to learn large margin halfspaces with RCN (under an additional anti-concentration assumption). 
The work of Blum {\em et al.}~\cite{BlumFKV96, BFK+:97} 
gave the first polynomial time algorithm for distribution-independent 
learning of halfspaces with RCN without any margin assumptions. 
Soon thereafter,~\cite{Cohen:97} gave a polynomial-time proper learning algorithm for the problem. 
Subsequently, Dunagan and Vempala~\cite{DunaganV04} gave a rescaled perceptron algorithm for solving
linear programs, which translates to a significantly simpler and faster proper learning algorithm.

The term ``Massart noise'' was coined after~\cite{Massart2006}.
An equivalent version of the model was previously studied by 
Rivest and Sloan~\cite{Sloan88, Sloan92, RivestSloan:94, Sloan96},
and a very similar asymmetric random noise model goes back to Vapnik~\cite{Vapnik82}.
Prior to this work, essentially no efficient algorithms with non-trivial error guarantees 
were known in the distribution-free Massart noise model. It should be noted that 
polynomial time algorithms with error $\opt+\eps$ 
are known~\cite{AwasthiBHU15, ZhangLC17, YanZ17} when the marginal distribution 
on the unlabeled data is uniform on the unit sphere. 
For the case that the unlabeled data comes from an isotropic log-concave distribution,
\cite{AwasthiBHZ16} give a \newblue{$d^{2^{\poly(1/(1-2\eta))}}/\poly(\eps)$} sample and time algorithm.

\subsection{Preliminaries} \label{ssec:prelims}
For $n \in \Z_+$, we denote $[n] \eqdef \{1, \ldots, n\}$.
We will use small boldface characters for vectors
and we let $\e_i$ denote the $i$-th vector of an orthonormal basis. 

For $\bx \in \R^d$, 
and $i \in [d]$, $\bx_i$ denotes the $i$-th coordinate of $\bx$, and 
$\|\bx\|_2 \eqdef (\littlesum_{i=1}^d \bx_i^2)^{1/2}$ denotes the $\ell_2$-norm
of $\bx$. We will use $\langle \bx, \by \rangle$ for the inner product between $\bx, \by \in \R^d$. 
We will use $\E[X]$ for the expectation of random variable $X$ and $\pr[\mathcal{E}]$
for the probability of event $\mathcal{E}$. 

An origin-centered halfspace is a Boolean-valued function $h_{\bw}: \R^d \to \{\pm 1\}$ 
of the form $h_{\bw}(\bx) = \sgn \left(\langle \bw, \bx \rangle  \right)$,
where $\bw \in \R^d$. (Note that we may assume w.l.o.g. that $\|\bw\|_2 =1$.) 
We denote by $\LTF_{d}$ 
the class of all origin-centered halfspaces on $\R^d$.

We consider a classification problem where labeled examples $(\bx,y)$ are drawn i.i.d. 
from a distribution $\D$. We denote by $\D_\bx$ the marginal of $\D$ on $\bx$, and for any $\bx$ 
denote $\D_y(\bx)$ the distribution of $y$ conditional on $\bx$. 
Our goal is to find a hypothesis classifier $h$ with low misclassification error. 
We will denote the misclassification error 
of a hypothesis $h$ with respect to $\D$ by $\err_{0-1}^{\D}(h) = \pr_{(\bx, y) \sim \D}[h(\bx) \neq y]$.
Let $\opt = \min_{h \in \LTF_{d}} \err_{0-1}^{\D}(h)$ denote the optimal misclassification 
error of any halfspace, and $\bw^{\ast}$ be the normal vector to a halfspace $h_{\bw^\ast}$ that achieves this.

\section{Algorithm for Learning Halfspaces with Massart Noise} \label{sec:alg}


In this section, we present the main result of this paper, which is an efficient algorithm 
that achieves $\eta+\eps$ misclassification error for distribution-independent 
learning of halfspaces with Massart noise, where $\eta$ is an upper bound on the noise rate. 


\yellow{Our algorithm uses (stochastic) gradient descent on a convex proxy function $L(\bw)$ for the misclassification error 
to identify a region with small misclassification error. The loss function penalizes the points which are misclassified 
by the threshold function $h_{\bw}$, proportionally to the distance from the corresponding hyperplane, while it rewards the correctly classified points at a smaller rate. Directly optimizing this convex objective does not lead to a separator with low error, but guarantees that for a non-negligible fraction of the mass away from the separating hyperplane the misclassification error 
will be at most $\eta + \eps$. By classifying points in this region according to the hyperplane and recursively working on the remaining points, we obtain an improper learning algorithm that achieves $\eta + \eps$ error overall.

}

We now develop some necessary notation before proceeding with the description and analysis of our algorithm.

Our algorithm considers the following convex proxy for the misclassification error 
\green{as a function of the weight vector $\bw$}: 
$$L(\bw) = \Ex_{(\bx,y) \sim \D} [ \LR_\lambda(-y \langle \bw, \bx \rangle )] \;,$$
under the constraint $\ltwo{\bw} \le 1$, where
$\LR_\lambda(z) = \left\{
	\begin{array}{ll}
		(1-\lambda) z  & \mbox{if } z \geq 0 \\
		\lambda z & \mbox{if } z < 0
	\end{array}
\right.$
and $\lambda$ is the \emph{leakage} parameter, which we will set to be $\lambda \approx \eta$.

We define the per-point misclassification error and the error of the proxy function as
$\err(\bw,\bx) = \pr_{y \sim \D_y(\bx)} [ \bw(\bx) \neq y ]$ and
$\ell(\bw,\bx) = \Ex_{y \sim \D_y(\bx)} [ \LR_\lambda(-y \dotp{\bw}{\bx} )]$ \green{respectively}.

Notice that $\err_{0-1}^{\D}(h_\bw) = \Ex_{\bx \sim \D_{\bx}} [ \err(\bw,\bx) ]$ and $L(\bw) = \Ex_{\bx \sim \D_{\bx}} [ \ell(\bw,\bx) ]$.
Moreover, $\opt = \Ex_{\bx \sim \D_\bx} [ \err({\bw^{\ast}},\bx) ] = \Ex_{\bx \sim \D_\bx} [ \eta(\bx) ]$.

\paragraph{Relationship between proxy loss and misclassification error} We first relate 
the proxy loss and the misclassification error:
\begin{claim}\label{claim:relationship}
  For any $\bw,\bx$, we have that $\ell(\bw,\bx) = (\err(\bw,\bx) - \lambda) | \dotp{ \bw}{ \bx } |$.
\end{claim}
\begin{proof}
We consider two cases:
\begin{itemize}[leftmargin=*]
  \item {\bf Case $\sgn(\langle \bw, \bx \rangle) = \sgn(\langle \bw^{\ast}, \bx \rangle)$}: In this case, we have that 
  $\err(\bw,\bx) = \eta(\bx)$, while
  $\ell(\bw,\bx) = \eta(\bx) (1-\lambda) |\langle \bw, \bx \rangle| - (1-\eta(\bx)) \lambda |\langle \bw, \bx \rangle| = (\eta(\bx) - \lambda) |\langle \bw, \bx \rangle|$.
  
  \item {\bf Case $\sgn(\langle \bw, \bx \rangle) \neq \sgn(\langle \bw^{\ast}, \bx \rangle)$}: In this case, we have that 
  $\err(\bw,\bx) = 1-\eta(\bx)$, while
  $\ell(\bw,\bx) = (1-\eta(\bx)) (1-\lambda) |\langle \bw, \bx \rangle| - \eta(\bx) \lambda |\langle \bw, \bx \rangle| = (1 - \eta(\bx) - \lambda) |\langle \bw, \bx \rangle|$.
\end{itemize}
This completes the proof of Claim~\ref{claim:relationship}.
\end{proof}

\noindent Claim~\ref{claim:relationship} shows that minimizing 
$\Ex_{\bx \sim \D_{\bx}} \left[ \frac{ \ell(\bw,\bx) } { | \dotp{ \bw}{ \bx } | } \right]$ 
is equivalent to minimizing the misclassification error. Unfortunately, this objective is hard 
to minimize as it is non-convex, but one would hope that minimizing $L(\bw)$ instead 
may have a similar effect. As we show in Section~\ref{sec:lb}, this is not true because $| \dotp{ \bw}{ \bx } |$ 
might vary significantly across points, and in fact it is not possible to 
use a convex proxy that achieves bounded misclassification error directly.

Our algorithm circumvents this difficulty by approaching the problem indirectly 
to find a non-proper classifier. Specifically, \newblue{our} algorithm works in multiple rounds, 
where within each round only points with high value of $| \dotp{ \bw}{ \bx } |$ are considered. 
The intuition is based on the fact that the approximation of the convex proxy 
to the misclassification error is more accurate for those points that have comparable distance to the halfspace.

\noindent In Section~\ref{ssec:alg-margin}, we handle the large margin case
and in Section~\ref{thm:margin-case} we handle the general case.

\subsection{Warm-up: Learning Large Margin Halfspaces} \label{ssec:alg-margin}

We consider the case that there is no probability mass within distance $\gamma$ 
from the separating hyperplane $\langle \bw^{\ast} ,\bx \rangle = 0$, $\|\bw^{\ast} \|_2=1$. Formally,
assume that for every $\bx\sim \D_\bx$, $\|\bx\|_2 \le 1$ and that $|\langle \bw^{\ast}, \bx \rangle| \ge \gamma$. 

The pseudo-code of our algorithm is given in Algorithm~\ref{alg:main-algorithm}. 
Our algorithm returns a decision list $[(\bw^{(1)},T^{(1)}), (\bw^{(2)},T^{(2)}), \cdots]$ as output. 
To classify a point $\bx$ given the decision list, the first $i$ is identified such that $|\langle \bw^{(i)}, \bx \rangle| \ge T^{(i)}$ and $\sgn(\langle \bw^{(i)}, \bx \rangle)$ is returned. If no such $i$ exists, an arbitrary prediction is returned.

\begin{algorithm}[H]
  \caption{Main Algorithm (with margin)}
  \label{alg:main-algorithm}
  \begin{algorithmic}[1]
    \State Set $S^{(1)} = \R^d$, $\lambda = \eta + \eps$, $m = \tilde {O}(\frac 1 {\gamma^2 \eps^4})$.
    \State Set $i \leftarrow 1$.
    \State \label{step:Dx-emp}\green{Draw $O\left((1/\eps^2) \log(1/(\eps \gamma))\right)$ samples from $\D_{\bx}$ 
    to form an empirical distribution $\tilde{\D}_{\bx}$.}
    \While{$\pr_{\bx \sim \green{\tilde{\D}_\bx}} \left[ \bx \in S^{(i)} \right] \ge \green{\eps} $}
      \State Set $\D^{(i)} = \D|_{S^{(i)}}$, the distribution conditional on the unclassified points.
      \State Let $L^{(i)}(\bw) = \Ex_{(\bx,y) \sim \D^{(i)}} [ \LR_\lambda(-y \langle \bw, \bx \rangle )]$
   \State \label{step:sgd}  Run SGD on $L^{(i)}(\bw)$ for $\green{\tilde{O}}(1/(\gamma^2 \eps^2))$ iterations to get $\bw^{(i)}$ 
   with $\|\bw^{(i)}\|_2 = 1$ such that $L^{(i)}(\bw^{(i)}) \le \min_{\bw: \|\bw\|_2 \leq 1} L^{(i)}(\bw) + \gamma \eps / 2$. 
      \State  Draw $m$ samples from $\D^{(i)}$ to form an empirical distribution $\D_m^{(i)}$.
      \State \label{step:thresholdestimation} Find a threshold $T^{(i)}$ such that 
      $\pr_{(\bx,y) \sim \D_m^{(i)}} [ |\langle \bw^{(i)}, \bx \rangle| \ge T^{(i)} ] \ge \gamma \eps$ 
      and the empirical misclassification error, 
      $\pr_{(\bx,y) \sim \D_m^{(i)}} [h_{\bw^{(i)}}(\bx) \neq y \, \big| \, |\langle \bw^{(i)}, \bx \rangle| \ge T^{(i)}]$, is minimized.
      \State Update the unclassified region $S^{(i+1)} \leftarrow S^{(i)} \setminus \{\bx : |\langle \bw^{(i)}, \bx \rangle| \ge T^{(i)} \} $ and set $i \leftarrow i + 1$.
    \EndWhile
    \State Return the classifier $[(\bw^{(1)},T^{(1)}), (\bw^{(2)},T^{(2)}), \cdots]$
  \end{algorithmic}
\end{algorithm}

\noindent The main result of this section is the following:

\begin{theorem}\label{thm:margin-case}
Let $\D$ be a distribution on $\B_d \times \{\pm 1 \}$ such that $\D_{\bx}$ satisfies the $\gamma$-margin
property with respect to $\bw^{\ast}$ and $y$ is generated by $\sgn(\langle \bw^{\ast} ,\bx \rangle)$ 
corrupted with Massart noise at rate $\eta<1/2$. Algorithm~\ref{alg:main-algorithm} uses 
$\tilde O(1/(\gamma^3 \eps^5))$ samples from $\D$, runs in  $\poly(d, 1/\eps, 1/\gamma)$ time,
and returns, with probability $2/3$, a classifier $h$ with misclassification error $\err_{0-1}^{\D}(h) \leq \eta+\eps$. 
\end{theorem}

\noindent Our analysis focuses on a single iteration of Algorithm~\ref{alg:main-algorithm}. 
We will show that a large fraction of the points is classified at every iteration within error $\eta+\eps$. 
To achieve this, we analyze the convex objective $L$. 
We start by showing that the optimal classifier $\bw^{\ast}$ obtains a \new{sufficiently small} negative objective value.

\begin{lemma}\label{lm:opt_val}
If $\lambda \ge \eta$, then $L(\bw^{\ast}) \le - \gamma (\lambda - \opt)$.
\end{lemma}
\begin{proof}
For any fixed $\bx$, {using Claim \ref{claim:relationship}}, we have that $$\ell(\bw^{\ast},\bx) = (\err(\bw^{\ast},\bx) - \lambda) | \langle \bw^{\ast}, \bx \rangle | = (\eta(\bx) - \lambda) |\langle \bw^{\ast}, \bx \rangle| \le - \gamma (\lambda - \eta(\bx)) \;,$$
since $|\langle \bw^{\ast}, \bx \rangle| \ge \gamma$ and $\eta(\bx) - \lambda \le 0$. 
Taking expectation over $\bx \sim \D_{\bx}$, the statement follows.
\end{proof}

\yellow{Lemma~\ref{lm:opt_val} is the only place where the Massart noise assumption is used in our approach
and establishes that points with sufficiently \new{small} negative value exist.
As we will show, any weight vector $\bw$ with this property can be found with few samples 
and must accurately classify some region of non-negligible mass away from it (Lemma~\ref{lm:structural}).}

We now argue that we can use stochastic gradient descent (SGD) to efficiently identify 
a point $\bw$ that achieves comparably small objective value to the guarantee 
of Lemma~\ref{lm:opt_val}. We use the following standard property of SGD:

\begin{lemma}[see, e.g., Theorem 3.4.11 in~\cite{Duchi16}]\label{lm:sgd_guarantees}
Let $L$ be any convex function. Consider the (projected) SGD iteration that is initialized at 
$\bw^{(0)} = \bold 0$ and for every step computes
$$\bw^{(t+\frac 1 2)} = \bw^{(t)} - \rho \bv^{(t)} \quad \text{ and } \quad 
\bw^{(t+1)} = \arg\min_{\bw: \ltwo{\bw} \le 1} \ltwo{ \bw - \bw^{(t+\frac 1 2)} } \;,$$ 
where $\bv^{(t)}$ is a stochastic gradient such that for all steps $\Ex[ \bv^{(t)} | \bw^{(t)} ] \in \partial L(\bw^{(t)}) $ 
and $\ltwo{\bv^{(t)}} \le 1$. 
Assume that SGD is run for $T$ iterations with step size $\rho =  \frac{1}{\sqrt{T}}$ and let $\bar \bw = \frac 1 T \sum_{t=1}^T \bw^{(t)}$. Then, for any $\eps, \delta > 0$, after $T  = \Omega(\log(1/\delta)/\eps^2)$ iterations with probability 
 with probability at least $1-\delta$ we have that
$L(\bar \bw) \le \min_{\bw: \ltwo{\bw} \le 1} L(\bw) + \eps$.
\end{lemma}

By Lemma \ref{lm:opt_val}, we know that $\min_{\bw: \|\bw\|_2 \le 1 } L(\bw) \le - \gamma (\lambda - \opt)$. 
By Lemma~\ref{lm:sgd_guarantees}, it follows that by running SGD on $L(\bw)$ with projection to the unit $\ell_2$-ball for 
$O\left(\log(1/\delta) / (\gamma^2(\lambda - \opt)^2) \right)$ steps, 
we find a $\bw$ such that $L(\bw) \le - \gamma (\lambda - \opt)/2$ \green{with probability at least $1-\delta$}. 

\noindent 

Note that we can assume 
without loss of generality that $\|\bw\|_2=1$, as increasing the magnitude of $\bw$ only decreases the objective value.

\medskip

We now consider the misclassification error of the halfspace $h_{\bw}$ conditional 
on the points that are further than some distance $T$ from the separating hyperplane. 
We claim that there exists a threshold $T>0$ where the restriction has non-trivial mass 
and the conditional misclassification error is small:

\begin{lemma}\label{lm:structural}
Consider a vector $\bw$ with $L(\bw) < 0$. There exists a threshold $T \ge 0$ such that 
(i) $\pr_{(\bx,y) \sim \D} [ |\langle \bw, \bx \rangle| \ge T ] \ge \frac{|L(\bw)|}{2 \lambda},$ and
(ii) $\pr_{(\bx,y) \sim \D} [h_\bw(\bx) \neq y \, \big| \, |\langle \bw, \bx \rangle| \ge T ] \le \lambda - \frac{|L(\bw)|}{2}.$
\end{lemma}

\begin{proof}
We will show there is a $T \ge 0$ such that 
$\pr_{(\bx,y) \sim \D} [h_\bw(\bx) \neq y \, \big| \, |\langle \bw, \bx \rangle| \ge T ] \le \lambda - \zeta$,
where $\zeta \eqdef |L(\bw)|/2$, or equivalently, 
$\Ex_{{\bx \sim \D_\bx}}[
  (\err(\bw,\bx) - \lambda  + \zeta)
  \Ind_{|\langle \bw, \bx \rangle| \ge T}] \le 0.$

\noindent For a $T$ drawn uniformly at random in $[0,1]$, we have that:
\begin{align*} 
\int_0^1 \Ex_{{\bx \sim \D_\bx}}[
  (\err(\bw,\bx) - \lambda +  \zeta) 
  1_{|\langle \bw, \bx \rangle| \ge T}] dT &= \E_{{\bx \sim \D_\bx}}[
  (\err(\bw,\bx) - \lambda) |\langle \bw, \bx \rangle|] + \zeta \E_{{\bx \sim \D_\bx}}[
  |\langle \bw, \bx \rangle|]\\
  &\le \E_{{\bx \sim \D_\bx}}[
  \ell(\bw,\bx)] + \zeta = L(\bw)  + \zeta = L(\bw) / 2 < 0 \;,
\end{align*}
\new{where the first inequality uses Claim~\ref{claim:relationship}.}
Thus, there exists a $\bar T$ such that 
$$\E_{{\bx \sim \D_\bx}}[(\err(\bw,\bx) - \lambda + \zeta) \Ind_{|\langle \bw, \bx \rangle| \ge \bar T}] \le 0 \;.$$ 
Consider the minimum such $\bar T$. Then we have
$$\int_{\bar T}^1 \E_{{\bx \sim \D_\bx}}[
    (\err(\bw,\bx) - \lambda + \zeta) 
    \Ind_{|\langle \bw, \bx \rangle| \ge T}] dT \ge - \lambda \cdot \pr_{\newblue{(\bx,y) \sim \D}}[ |\langle \bw, \bx \rangle| \ge \bar T ] \;.$$
\green{By definition of $\bar T$, it must be the case that 
$$\int_0^{\bar T} \E_{{\bx \sim \D_\bx}}[
    (\err(\bw,\bx) - \lambda + \zeta) 
    \Ind_{|\langle \bw, \bx \rangle| \ge T}] dT \ge 0 \;.$$
Therefore, 
$$\frac{L(\bw)}{2}\ge\int_{\bar T}^1 \E_{{\bx \sim \D_\bx}}[
    (\err(\bw,\bx) - \lambda + \zeta) 
    \Ind_{|\langle \bw, \bx \rangle| \ge T}] dT \ge - \lambda \cdot \pr_{\newblue{(\bx,y) \sim \D}}[ |\langle \bw, \bx \rangle| \ge \bar T ] \;,$$ 
which implies that $\pr_{\newblue{(\bx,y) \sim \D}} [|\langle \bw, \bx \rangle| \ge \bar T ] \ge \frac{|L(\bw)|}{2\lambda} $. This completes the proof of Lemma~\ref{lm:structural}.
}

\end{proof}

Even though minimizing the convex proxy $L$ does not lead to low misclassification error overall, 
Lemma~\ref{lm:structural} shows that there exists a region of non-trivial mass where it does. 
This region is identifiable by a simple threshold rule.
We are now ready to prove Theorem \ref{thm:margin-case}.

\begin{proof}[Proof of Theorem \ref{thm:margin-case}]
We consider the steps of Algorithm~\ref{alg:main-algorithm} in each iteration of the while loop. 
At iteration $i$, we consider a distribution $\D^{(i)}$ consisting only of points not handled in previous iterations. 

We start by noting that with high probability the total number of iterations is \newblue{$\tilde{O}(1/(\gamma \epsilon))$}.
This can be seen as follows: The empirical probability mass under $\D_m^{(i)}$ 
of the region $\{\bx : |\langle \bw^{(i)}, \bx \rangle| \ge T^{(i)} \}$ removed from $S^{(i)}$ 
to obtain $S^{(i+1)}$ is at least $\gamma \eps$ (Step~\ref{step:thresholdestimation}).
Since  $m = \tilde {O}(1/(\gamma^2 \eps^4))$, the DKW inequality~\cite{DKW56} implies that the true 
probability mass of this region is at least $\gamma \eps/2$ with high probability. 
By a union bound over $i \leq K = \Theta(\log(1/\eps)/(\eps\gamma))$, it follows
that with high probability we have that $\Pr_{\D_{\bx}}[S^{(i+1)}] \leq (1-\gamma \eps/2)^i$ for all $i \in [K]$.
After $K$ iterations, we will have that $\Pr_{\D_{\bx}}[S^{(i+1)}] \leq \eps/3$.
Step~\ref{step:Dx-emp} guarantees that the mass of $S^{(i)}$ under $\tilde{\D}_{\bx}$
is within an additive $\eps/3$ of its mass under $\D_{\bx}$, for $i \in [K]$. This implies that
the loop terminates after at most $K$ iterations with high probability.

By Lemma~\ref{lm:opt_val} and the fact that every $\D^{(i)}$ has margin $\gamma$, it follows 
that the minimizer of the loss $L^{(i)}$ has value less than 
$-\gamma (\lambda - \opt^{(i)}) \le - \gamma \eps$, as $\opt^{(i)} \le \eta$ and $\lambda = \eta + \eps$.
By the guarantees of Lemma~\ref{lm:sgd_guarantees}, running SGD in line~\ref{step:sgd} on $L^{(i)}( \cdot )$ with projection 
to the unit $\ell_2$-ball for $O\left(\log(1/\delta)/(\gamma^2 \eps^2)\right)$ steps, 
we obtain a $\bw^{(i)}$ such that, with probability at least $1-\delta$, it holds 
$L^{(i)}(\bw^{(i)}) \le - \gamma \eps/2$ and $\|\bw^{(i)}\|_2 = 1$. 
Here $\delta>0$ is a parameter that is selected so that the following claim 
holds: With probability at least $9/10$, for all iterations $i$ of the while loop
we have that $L^{(i)}(\bw^{(i)}) \le - \gamma \eps/2$. Since the total number of iterations
is \newblue{$\tilde{O}(1/(\gamma \epsilon))$}, setting $\delta$
to $\tilde{\Omega}(\eps \gamma)$ and applying a union bound over all iterations gives the previous claim. 
Therefore, the total number of SGD steps per iteration is $\tilde{O}(1 / (\gamma^2\eps^2))$. 
For a given iteration of the while loop, running SGD requires $\tilde{O}(1 / (\gamma^2\eps^2))$ 
samples from $\D^{(i)}$ which translate to at most $\tilde{O}\left(1/(\gamma^2 \eps^3)\right)$ samples from $\D$, 
as $\pr_{\bx \sim \D_\bx} \left[ \bx \in S^{(i)} \right]  \geq 2\eps/3$.  

Lemma~\ref{lm:structural} implies that there exists $T \ge 0$ such that:
\begin{enumerate}
\item[(a)] $\pr_{(\bx,y) \sim \D^{(i)}} [ |\langle \bw, \bx \rangle| \ge T ] \ge \gamma \eps,$ and
\item[(b)] $\pr_{(\bx,y) \sim \D^{(i)}} [h_\bw(\bx) \neq y \, \big| \, |\langle \bw, \bx \rangle| \ge T ] \le \eta + \eps.$
\end{enumerate}
Line~\ref{step:thresholdestimation} of  Algorithm~\ref{alg:main-algorithm} estimates the threshold using samples. By the 
DKW inequality~\cite{DKW56}, we know that with $m = \tilde {O}(1/(\gamma^2 \eps^4))$ samples we can estimate the CDF 
within error $\gamma \eps^2$ with probability $1-\poly(\eps,\gamma)$. This suffices to estimate the probability 
mass of the region within additive $\gamma \eps^2$ and the misclassification error within $\eps/3$. 
This is satisfied for all iterations with constant probability.

In summary, with high constant success probability, Algorithm~\ref{alg:main-algorithm} 
runs for \newblue{$\tilde{O}(1/(\gamma \epsilon))$} iterations and draws 
$\tilde {O}(1/(\gamma^2 \eps^4))$ samples per round for a total of 
$\tilde {O}(1/(\gamma^3 \eps^5))$ samples. As each iteration runs in polynomial time, 
the total running time follows.

When the while loop terminates, we have that $\pr_{\bx \sim \D_\bx} [ \bx \in S^{(i)}] \le 4\eps/3$, 
i.e., we will have accounted for at least a $(1-4\epsilon/3)$-fraction of the total probability mass. 
Since our algorithm achieves misclassification error at most $\eta + 4\eps/3$ in all the regions we accounted for, 
its total misclassification error is at most $\eta + 8\eps/3$. 
Rescaling $\eps$ by a constant factor gives Theorem~\ref{thm:margin-case}.
\end{proof}

\begin{remark}\label{rem:margin-error}
 If the value of $\opt$ is smaller than $\eta - \xi$ for some value $\xi > 0$, 
 Algorithm~\ref{alg:main-algorithm} gets misclassification error less 
 than $\eta - \Omega( \gamma^2 \xi^2 )$ when run for $\eps = O(\gamma^2 \xi^2)$. 
This is because, in the first iteration, $L^{(1)}(\bw^{(1)}) \le - \gamma (\lambda - \opt) / 2 \le - \gamma \xi / 2$, 
which implies, by Lemma~\ref{lm:structural}, that the obtained error in $S^{(1)}$ 
is at most $\lambda - \gamma \xi / 4$. The misclassification error in the remaining regions 
is at most $\lambda + \eps$, and region $S^{(1)}$ has probability mass at least $\gamma \xi / 4$. 
Thus, the total misclassification error is at most 
$\lambda + \eps - \gamma^2 \xi^2 / 16 = \eta - \Omega( \gamma^2 \xi^2 )$, 
when run for $\eps = O(\gamma^2 \xi^2)$.
\end{remark}

\subsection{The General Case} \label{ssec:alg-general}

In the general case, we assume that $\D_\bx$ is an arbitrary distribution supported on $b$-bit integers. 
While such a distribution might have exponentially small margin in the dimension $d$ (or even $0$), 
we will preprocess the distribution to ensure a margin condition by removing outliers.

We will require the following notion of an outlier:

\begin{definition}[\cite{DV:04}]
We call a point $\bx$ in the support of a distribution $\D_\bx$ a $\beta$-outlier, 
if there exists a vector $\bw \in \R^d$ such that
$\dotp{\bw}{\bx}^2  \geq \beta \Ex_{\bx \sim \D_\bx}[ \dotp{\bw}{\bx}^2 ].$
\end{definition}

We will use Theorem~3 of~\cite{DV:04}, 
which shows that any distribution supported on $b$-bit integers can be efficiently 
preprocessed using samples so that no large outliers exist.

\begin{lemma}[Rephrasing of Theorem 3 of \cite{DV:04}]\label{lm:outlier}
Using $m = \tilde{O}(d^2 b)$ samples from $\D_\bx$, one can identify with high probability an ellipsoid $E$ such that $\pr_{\bx \sim \D_\bx} [ \bx \in E ] \ge \frac 1 2$ and $\D_\bx |_E$ has no $\Gamma^{-1} = \tilde{O}(d b)$-outliers.
\end{lemma}

Given this lemma, we can adapt Algorithm~\ref{alg:main-algorithm} for the large margin case to work in general. 
The pseudo-code is given in Algorithm~\ref{alg:main-algorithm-general}. 
It similarly returns a decision list [$(\bw^{(1)},T^{(1)},E^{(1)})$, $(\bw^{(2)},T^{(2)},E^{(2)})$, $\cdots$] as output. 

\begin{algorithm}[H]
  \caption{Main Algorithm (general case)}
  \label{alg:main-algorithm-general}
  \begin{algorithmic}[1]
    \State Set $S^{(1)} = \R^d$, $\lambda = \eta + \eps$, $\Gamma^{-1} = \tilde{O}(d b)$, $m = \tilde {O}(\frac 1 {\Gamma^2 \eps^4})$.
    \State Set $i \leftarrow 1$.
    \State \label{step:Dx-emp-gen}\green{Draw $O\left((1/\eps^2) \log(1/(\eps \Gamma))\right)$ samples from $\D_{\bx}$ 
    to form an empirical distribution $\tilde{\D}_{\bx}$.}
    \While{$\pr_{\bx \sim \green{\tilde{\D}_{\bx}}} \left[ \bx \in S^{(i)} \right] \ge \green{\eps} $}
      \State Run the algorithm of Lemma~\ref{lm:outlier} to remove $\Gamma^{-1}$-outliers from the distribution $\D_{S^{(i)}}$ by filtering points outside the ellipsoid $E^{(i)}$.
      \State Let $\Sigma^{(i)} = \Ex_{(\bx,y) \sim \D^{(i)}|_{S^{(i)}}} [ \bx \bx^T ]$ and set $\D^{(i)} = \Gamma \Sigma^{(i)-1/2} \cdot \D|_{S^{(i)} \cap E^{(i)}}$ be the distribution $\D|_{S^{(i)} \cap E^{(i)}}$ brought in isotropic position and rescaled by $\Gamma$ so that all vectors have \new{$\ell_2$-norm at most $1$}.
      
      \State Let $L^{(i)}(\bw) = \Ex_{(\bx,y) \sim \D^{(i)}} [ \LR_\lambda(-y  \dotp{ \bw}{ \bx } )]$
      \State \label{step:sgd-general} Run SGD on $L^{(i)}(\bw)$ for $\green{\tilde{O}}(1/(\Gamma^2 \eps^2))$ iterations, 
      to get $\bw^{(i)}$ with $\|\bw^{(i)}\|_2 = 1$ such that $L^{(i)}(\bw^{(i)}) \le \min_{\bw: \|\bw\|_2 \leq 1} L^{(i)}(\bw) + \Gamma \eps / 2$.
      \State Draw $m$ samples from $\D^{(i)}$ to form an empirical distribution $\D_m^{(i)}$.
      \State \label{step:thresholdestimation-gen} Find a threshold $T^{(i)}$ such that $\pr_{(\bx,y) \sim \D_m^{(i)}} [ |\langle \bw^{(i)}, \bx \rangle| \ge T^{(i)} ] \ge \Gamma \eps$ and the empirical misclassification error, $\pr_{(\bx,y) \sim \D_m^{(i)}} [h_\bw(\bx) \neq y \, \big| \, |\langle \bw^{(i)}, \bx \rangle| \ge T^{(i)} ]$, is minimized.
      \State Revert the linear transformation by setting $\bw^{(i)} \leftarrow \Gamma \Sigma^{(i)-1/2} \bw^{(i)}$.
      \State Update the unclassified region $S^{(i+1)} \leftarrow S^{(i)} \setminus \{\bx : \bx \in E^{(i)} \wedge |\langle \bw^{(i)}, \bx \rangle| \ge T^{(i)} \} $ and set $i \leftarrow i + 1$.
    \EndWhile
    \State Return the classifier $[(\bw^{(1)},T^{(1)},E^{(1)}), (\bw^{(2)},T^{(2)},E^{(2)}), \cdots]$
  \end{algorithmic}
\end{algorithm}

\noindent Our main result is the following theorem:
\begin{theorem}\label{thm:general-case}
Let $\D$ be a distribution over $(d+1)$-dimensional labeled examples with bit-complexity $b$, 
generated by an unknown halfspace corrupted by Massart noise at rate $\eta<1/2$. 
Algorithm~\ref{alg:main-algorithm-general} uses $\tilde O(d^3 b^3/\eps^5)$ samples, 
runs in $\poly(d, 1/\eps, b)$ time, and returns, with probability $2/3$, a classifier $h$ with 
misclassification error $\err_{0-1}^{\D}(h) \leq \eta+\eps$.
\end{theorem}

We now analyze Algorithm~\ref{alg:main-algorithm-general} and establish Theorem~\ref{thm:general-case}. 
To do this, we need to adapt Lemma~\ref{lm:opt_val} to the case without margin. 
We replace the \new{margin} condition by requiring that the minimum eigenvalue of the covariance 
matrix is at least $\Gamma$.

\begin{lemma}\label{lm:opt_val_general}
Let $\D_\bx$ be any distribution over points with $\ell_2$-norm bounded by 1, with covariance having minimum eigenvalue at least $\Gamma$. If $\lambda \ge \eta$, then $\min_{\bw: \ltwo{\bw} \le 1} L(\bw) \le - \Gamma (\lambda - \eta)$.
\end{lemma}
\begin{proof}
We will show the statement for the \new{optimal} unit vector $\bw^{\ast}$.
For any fixed $\bx$, we have that $$\ell(\bw^{\ast},\bx) = (\err(\bw^{\ast},\bx) - \lambda) | \langle \bw^{\ast}, \bx \rangle | = (\eta(\bx) - \lambda) |\langle \bw^{\ast}, \bx \rangle| \le -  (\lambda - \eta) |\langle \bw^{\ast}, \bx \rangle|.$$
Taking expectation over $\bx$ drawn from $\D_\bx$, we get the statement as 
$$\Ex[|\langle \bw^{\ast}, \bx \rangle|] \ge \Ex[|\langle \bw^{\ast}, \bx \rangle|^2]  \ge \Gamma,$$
where we used the fact that for all points $\bx$, $|\langle \bw^{\ast}, \bx \rangle| \le \ltwo{\bx}^2 \le 1$.
\end{proof}

With Lemma~\ref{lm:opt_val_general} in hand, we are ready to prove Theorem~\ref{thm:general-case}. We will use Lemma~\ref{lm:sgd_guarantees} and Lemma~\ref{lm:structural} whose statements do not require that the distribution 
of points has large margin.

\begin{proof}[Proof of Theorem \ref{thm:general-case}]
We again consider the steps of Algorithm~\ref{alg:main-algorithm-general} in every iteration $i$. At every iteration, we consider a distribution $\D^{(i)}$ consisting only of points not handled in previous iterations. 

Similar to the proof of Theorem~\ref{thm:margin-case}, we start by noting that with high probability the total number of iterations is $\tilde{O}(1/(\Gamma \epsilon))$. This is because at every iteration, the empirical probability mass under $\D_m^{(i)}$ 
of the region $\{\bx : |\langle \bw^{(i)}, \bx \rangle| \ge T^{(i)} \}$ removed from $S^{(i)}$ 
to obtain $S^{(i+1)}$ is at least $\Gamma \eps$ and thus by the DKW inequality~\cite{DKW56} implies the true 
probability mass of this region is at least $\Gamma \eps/2$ with high probability.
After $K = \Theta(\log(1/\eps)/(\eps\Gamma))$ iterations, we will have that $\Pr_{\D_{\bx}}[S^{(i+1)}] \leq \eps/3$.
Step~\ref{step:Dx-emp-gen} guarantees that the mass of $S^{(i)}$ under $\tilde{\D}_{\bx}$
is within an additive $\eps/3$ of its mass under $\D_{\bx}$, for $i \in [K]$. This implies that
the loop terminates after at most $K$ iterations with high probability.

At every iteration, the distribution $\D^{(i)}$ is rescaled so that the norm of all points is bounded by $1$ and the covariance matrix has minimum eigenvalue $\Gamma$ as guaranteed by Lemma~\ref{lm:outlier}. By Lemma~\ref{lm:opt_val_general}, it follows 
that the minimizer of the loss $L^{(i)}$ has value less than 
$- \Gamma (\lambda - \eta) \le - \Gamma \eps$.
By the guarantees of Lemma~\ref{lm:sgd_guarantees}, running SGD in line~\ref{step:sgd-general} on $L^{(i)}( \cdot )$ with projection 
to the unit $\ell_2$-ball for $O\left(\log(1/\delta)/(\Gamma^2 \eps^2)\right)$ steps, 
we obtain a $\bw^{(i)}$ such that, with probability at least $1-\delta$, it holds 
$L^{(i)}(\bw^{(i)}) \le - \Gamma \eps/2$ and $\|\bw^{(i)}\|_2 = 1$. 
Here $\delta>0$ is a parameter that is selected so that the following claim 
holds: With probability at least $9/10$, for all iterations $i$ of the while loop
we have that $L^{(i)}(\bw^{(i)}) \le - \Gamma \eps/2$. Since the total number of iterations
is \newblue{$\tilde{O}(1/(\Gamma \epsilon))$}, setting $\delta$
to $\tilde{\Omega}(\eps \Gamma)$ and applying a union bound over all iterations gives the previous claim. 
Therefore, the total number of SGD steps per iteration is $\tilde{O}(1 / (\Gamma^2\eps^2))$. 
For a given iteration of the while loop, running SGD requires $\tilde{O}(1 / (\Gamma^2\eps^2))$ 
samples from $\D^{(i)}$ which translate to at most $\tilde{O}\left(1/(\Gamma^2 \eps^3)\right)$ samples from $\D$, 
as $\pr_{\bx \sim \D_\bx} \left[ \bx \in S^{(i)} \right]  \geq 2\eps/3$.  

Then, similar to the proof of Theorem~\ref{thm:margin-case}, Lemma~\ref{lm:structural} implies that there exists a threshold $T \ge 0$, such that:
\begin{itemize}
  \item[(a)] $\pr_{(\bx,y) \sim \D^{(i)}} [ |\langle \bw, \bx \rangle| \ge T ] \ge \Gamma \eps,$ and
  \item[(b)] $\pr_{(\bx,y) \sim \D^{(i)}} [h_\bw(\bx) \neq y \, \big| \, |\langle \bw, \bx \rangle| \ge T ] \le \eta + \eps.$
\end{itemize}

Line~\ref{step:thresholdestimation-gen} of  Algorithm~\ref{alg:main-algorithm-general}  estimates the threshold using samples. By the 
DKW inequality~\cite{DKW56}, we know that with $m = \tilde {O}(\frac 1 {\Gamma^2 \eps^4})$ samples we can estimate the CDF 
within error $\Gamma \eps^2$ with probability $1-\poly(\eps,\Gamma)$. This suffices to estimate the probability 
mass of the region within additive $\Gamma \eps^2$ and the misclassification error within $\eps/3$. 
This is satisfied for all iterations with constant probability.

In summary, with high constant success probability, Algorithm~\ref{alg:main-algorithm-general} 
runs for \newblue{$\tilde{O}(1/(\Gamma \epsilon))$} iterations and draws 
$\tilde {O}(1/(\Gamma^2 \eps^4))$ samples per round for a total of 
$\tilde {O}(1/(\Gamma^3 \eps^5))$ samples. As each iteration runs in polynomial time, 
the total running time follows.

When the while loop terminates, we have that $\pr_{\bx \sim \D_\bx} [ \bx \in S^{(i)}] \le 4\eps/3$, 
i.e., we will have accounted for at least a $(1-4\epsilon/3)$-fraction of the total probability mass. 
Since our algorithm achieves misclassification error at most $\eta + 4\eps/3$ in all the regions we accounted for, 
its total misclassification error is at most $\eta + 8\eps/3$. 
Rescaling $\eps$ by a constant factor gives Theorem~\ref{thm:general-case}.
\end{proof}

\section{Lower Bounds Against Natural Approaches} \label{sec:lb}

In this section, we show that certain natural approaches for learning halfspaces
with Massart noise inherently fail, even in the large margin case.

We begin in Section~\ref{min_surrogate} by showing that the common 
approach of using a convex surrogate function for the 0-1 loss cannot lead 
to non-trivial misclassification error. (We remark that this comes in sharp contrast 
with the problem of learning large margin halfpaces with RCN, where a convex surrogate
works, see, e.g., Theorem~\ref{thm:rnc} in Section~\ref{sec:rcn}). 

In Section~\ref{ssec:thresh}, we provide evidence that improving the misclassification guarantee of $\eta+\eps$ 
achieved by our algorithm requires a genuinely different approach. In particular, we show that the approach 
of iteratively using {\em any} convex proxy followed by thresholding gets stuck 
at error $\Omega(\eta)+\eps$, even in the large margin case.

\subsection{Lower Bounds Against Minimizing a Convex Surrogate Function}\label{min_surrogate}
\new{One of the most common approaches in machine learning is to replace the 0-1 loss
in the ERM by an appropriate convex surrogate and solve the corresponding convex optimization problem.
In this section, we show that this approach inherently fails to even give 
a weak learner in the presence of Massart noise --- even under a margin assumption.} 

In more detail, we construct distributions over a finite sets of points in the two-dimensional unit ball 
for which the method of minimizing a convex surrogate will always 
\new{have misclassification error $\min\{1/2,\Theta(\eta/\gamma)\}$}, 
where $\gamma$ is the maximum margin with respect to any hyperplane. 
\new{Our proof is inspired by an analogous construction in~\cite{LongS10}, 
which shows that one cannot achieve non-trivial misclassification error 
for learning halfspaces in the presence of RCN, using certain convex boosting techniques.} 
\new{Our argument is more involved in the sense that we need to distinguish 
two cases and consider different distributions for each one. Furthermore, by leveraging the additional strength of the Massart noise model, we are able to show that the misclassification error has to 
be larger than the noise level $\eta$ by a factor of $1/\gamma$.} 

In particular, our first case corresponds to \new{the situation where} the convex surrogate function 
is such that misclassified points are penalized by a fair amount and therefore the effect of noise 
of correctly classified points on the gradient is significant. This allows a significant amount of 
probability mass to be in the region where the true separating hyperplane and the one defined 
by the minimum of the convex surrogate function disagree. The second case, which is the 
complement of the first one, uses the fact that the contribution of a correctly classified point 
on the gradient is not much smaller than that of a misclassified point, again allowing a 
significant amount of probability mass to be given to the aforementioned disagreement region.   
Formally, we prove the following: 

\begin{theorem} \label{thm:lb-surr}
Consider the family of algorithms that produce a classifier 
$\sgn(\langle\bw^{\ast}, \bx\rangle)$, where $\bw^{\ast}$ is the minimum 
of the function $G(\bw)=\E_{(\bx,y)\sim \D}[\phi(y\langle \bw, \bx\rangle)]$. 
For any decreasing convex \footnote{\new{The function is not necessarily differentiable. In case it is not, 
being \emph{convex} means that the sub-gradients of the points are monotonically non-decreasing.}} 
function $\phi: \mathbb{R}\rightarrow \mathbb{R}$, there exists a distribution $\D$ over 
$\B_2 \times \{ \pm 1\}$with margin $\gamma\leq \frac{\sqrt{3}-1}{4}$ such that the classifier 
$\sgn(\langle\bw^{\ast}, \bx\rangle)$, misclassifies a $\min\{\frac{\eta}{8\gamma},\frac12\}$ fraction of the points.      
\end{theorem}

\begin{proof}
We consider algorithms that perform ERM with a convex surrogate, i.e., 
minimize a loss of the form
$G(\bw)=\E_{(\bx,y)\sim \D}[\phi(y\langle \bw, \bx\rangle)]$, 
for some convex function $\phi: \mathbb{R}\rightarrow \mathbb{R}$ 
for \new{$\|\bw\|_2\leq 1$}.
\new{
We can assume without loss of generality that $\phi$ is differentiable and its derivative is non-decreasing. 
Even if there is a countable number of points in which it is not,  there is a subderivative 
that we can pick for each of those points such that the derivative is increasing overall, 
since we have assumed that $\phi$ is convex. Therefore, our argument still goes through even without
assuming differentiability.}

We start by calculating the gradient of $G$ as a function of the derivative of $\phi$ at the minimum of $G$. 
Suppose that $ \bv\in \mathbb{R}^d$ is the minimizer of $G$ subject to \new{$\|\bw\|_2\leq 1$}. 
This requires that \new{either $\nabla G(\bv)$ is parallel to $\bv$, in case the unconstrained minimum 
lies outside the region $\|\bw\|_2\leq 1$, or $\nabla G(\bv)=\mathbf{0}$.} 
Therefore, we have that \new{for every $i>1$, the following holds}:
\[
\frac{\partial G}{\partial \bw_i}(\bv)=\E_{(\bx,y)\sim \D}[\phi^\prime(y\langle \bv, \bx\rangle)(y \bx_i)]=0 \;.
\]
\new{Our lower bound construction produces a distribution $\D$ over $(\bx, y)$ 
whose $\bx$ marginal, $\D_{\bx}$, is supported on the $2$-dimensional unit ball.}
We need to consider two complementary cases for the convex function $\phi$.
For each case, we will define judiciously chosen distributions, $\D_1,\D_2$ for which the result holds. 

\paragraph{Case I:} There exists $z\in [0,\sqrt{3}/2]$ such that:
$|\phi^\prime(z)|<\frac12 \frac{\eta}{1-\eta}|\phi^\prime(-z)|$.

\smallskip

\noindent In this case, we consider the distribution shown in Figure~\ref{fg:case1} (left), 
where the point $(z,-\gamma)$ has probability mass $p$ and 
the remaining $1-p$ mass in on the point $(z,\sqrt{1-z^2})$. 
We need to pick the parameter $p$ so that $ \bv=\be_1$ is the minimum of $G(\bw)$.    

Note that the misclassification error is 
$\new{\err_{0-1}^{\D_1} (\sgn(\langle\bv,\bx\rangle))} = p+(1-p)\cdot \eta$.
The condition that $\bv=\be_1$  is
a minimizer of $G(\bw)$ is equivalent to 
$\E_{(\bx,y) \sim \D_1}[\phi^\prime(y\langle\bv, \bx\rangle)(y \bx_2)]=0$.
Substituting for our choice of $\D_1$ 
with noise level $\eta$ on $(z,-\gamma)$ and $0$ on $(z,\sqrt{1-z^2})$, we get:
\[
p\cdot\phi^\prime(-z)\cdot\gamma+(1-p)\cdot(1-\eta)\phi^\prime(z)\cdot \sqrt{1-z^2} 
+(1-p)\cdot\eta\cdot \phi^\prime(-z)\cdot (-\sqrt{1-z^2})=0 \;.
\]
Equivalently, we have:
\[
(1-p)\cdot\eta\cdot |\phi^\prime(-z)|\cdot  \sqrt{1-z^2}= 
p\cdot\gamma\cdot|\phi^\prime(-z)|+(1-p)\cdot(1-\eta)|\phi^\prime(z)| \sqrt{1-z^2} \;.
\]
Now, suppose that $|\phi^\prime(z)|=(1-\alpha) \frac{\eta}{1-\eta}|\phi^\prime(-z)|$, for some $\alpha>\frac12$. 
By substituting and simplifying, we get:
$$ p\cdot\gamma=\alpha(1-p)\eta \sqrt{1-z^2}=(1-p)\eta\Delta \;,
$$
where $\Delta=\alpha\sqrt{1-z^2}$, which in turns gives that
$$p=\frac{\eta\Delta}{\gamma+\eta\Delta} \;.$$ 
Thus, the misclassification error is
\[\new{\err_{0-1}^{\D_1} (\sgn(\langle\bv,\bx\rangle))}=p+(1-p)\eta=\eta+(1-\eta)p=\eta+\frac{(1-\eta)\eta\Delta}{\gamma+\eta\Delta}=\frac{\eta(\gamma+\Delta)}{\gamma+\eta\Delta}\geq \frac{1}{1+\frac{\gamma}{\eta\Delta}} \;.\] 
Note that for margin $\gamma\leq \eta\cdot \Delta$, we have that 
$\new{\err_{0-1}^{\D_1} (\sgn(\langle\bv,\bx\rangle))}\geq \frac12$, and we can achieve error exactly $\frac12$ by setting the point $Q_1$ 
at distance exactly $\eta\cdot \Delta$. On the other hand, when the margin is 
$\gamma\leq \eta\cdot\Delta$, we have: 
$\new{\err_{0-1}^{\D_1} (\sgn(\langle\bv,\bx\rangle))} \geq \frac{\eta\Delta}{2\gamma}\geq \frac{\eta}{8\gamma}$. 
The last inequality comes from the fact that $\Delta=\alpha\sqrt{1-z^2}\geq 1/4$, since $\alpha\geq 1/2$ and $z\leq \sqrt{3}/2$.
  
%
%
%
%
%

\begin{figure}
\begin{center}
\includegraphics[scale=0.2]{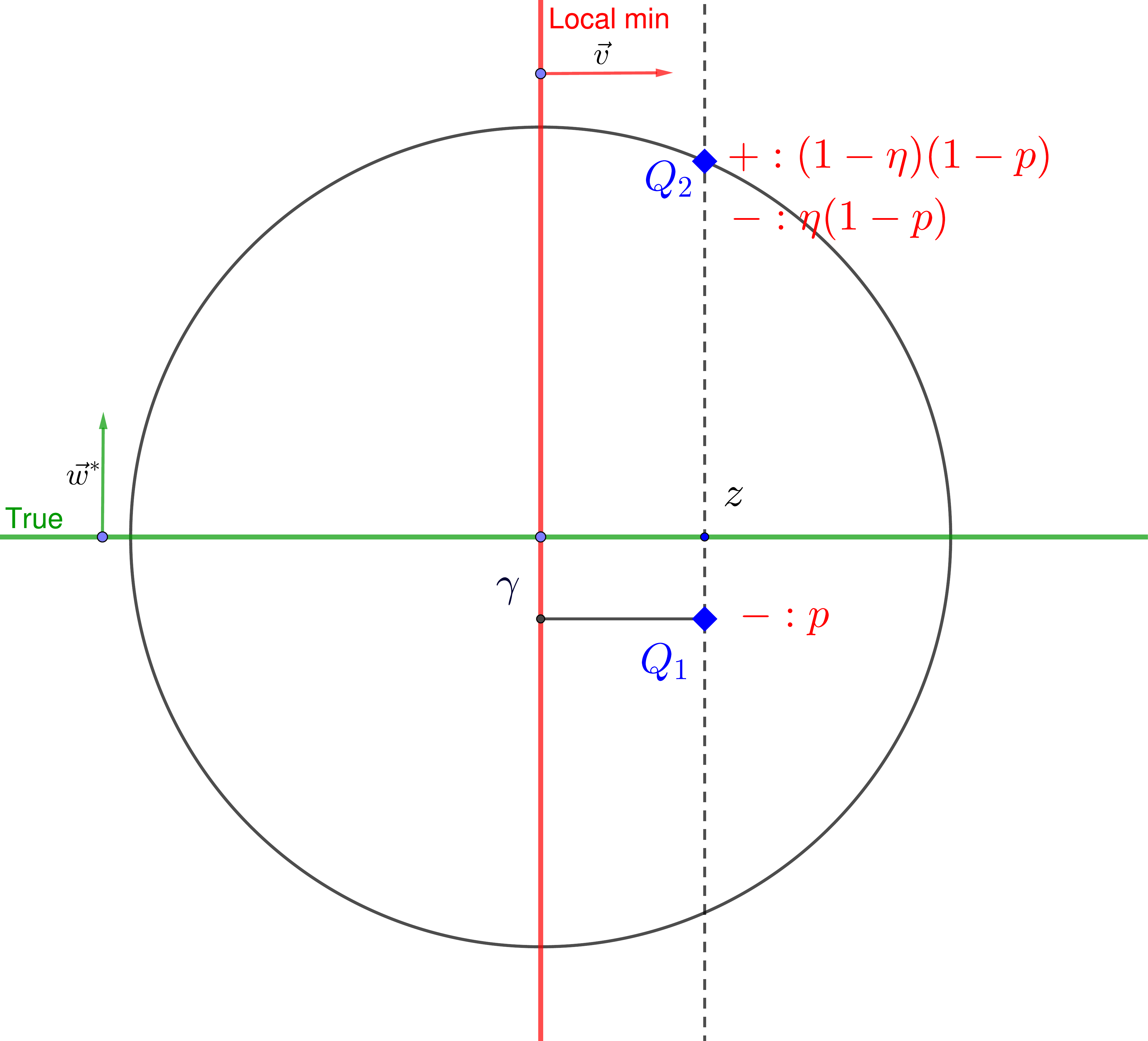}
\includegraphics[scale=0.3]{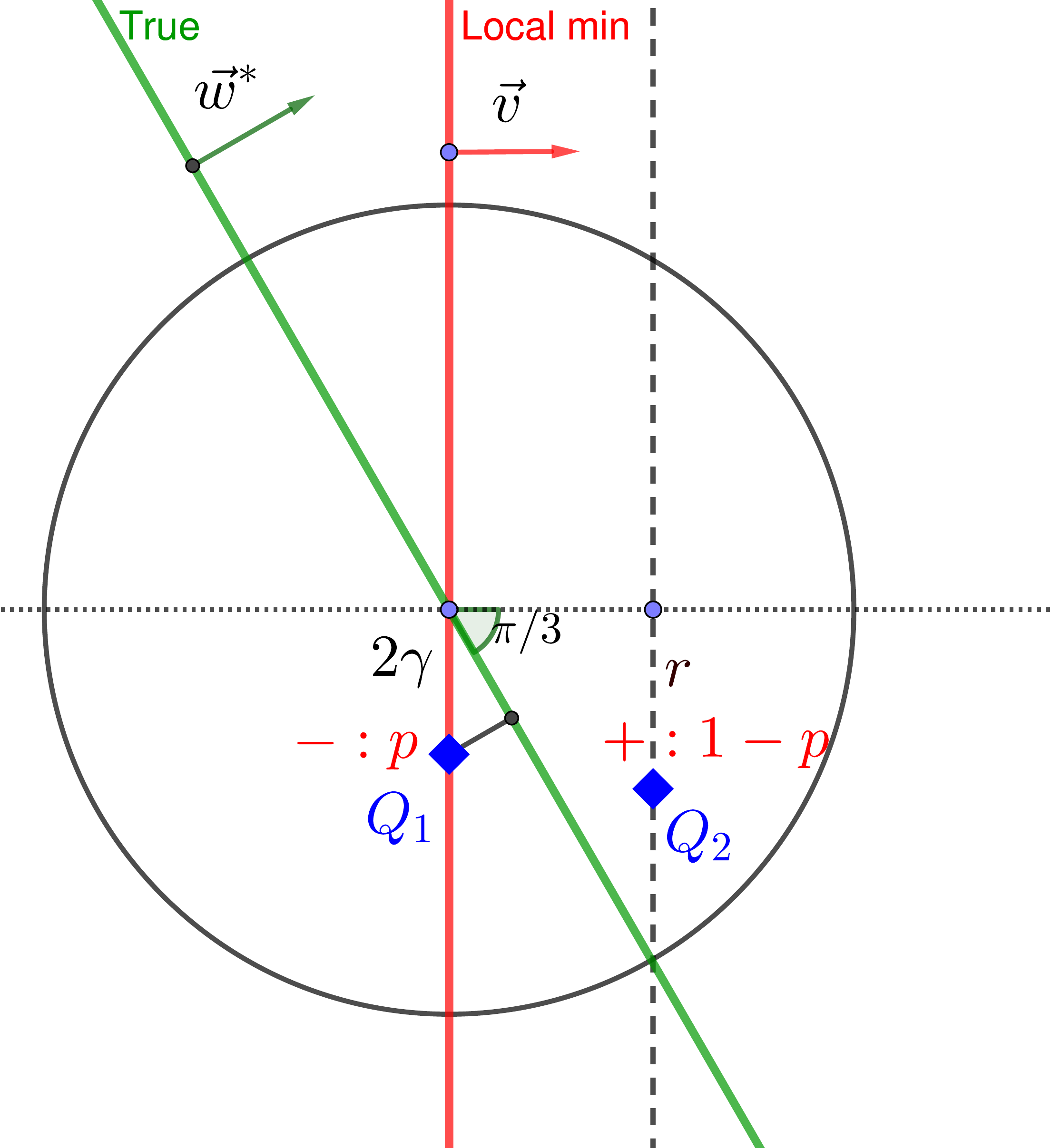}
\end{center}
\caption{Probability distribution for Case I 
is on the left and for the complementary Case II 
 is on the right.}
 \label{fg:case1}
\end{figure}

\paragraph{Case II:} For all $z\in [0,\sqrt{3}/2]$ we have that  $|\phi^\prime(z)|\geq\frac12 \frac{\eta}{1-\eta}|\phi^\prime(-z)|$.


\smallskip

\noindent In this case, we consider the distribution shown in Figure~\ref{fg:case1} (right), 
where the only points that have non-zero mass are: $(0,-2\gamma)$, which has probability mass $p$, 
and $(1/2,-r)$, with mass $1-p$. We need to appropriately select the parameters $p$ and $r$, 
so that $\bv$ is actually the minimizer of the function $G(\bw)$, and the misclassification error 
(which is equal to $p$ in this case) is maximized.  

Note that $\bv$ satisfies $\E_{(\bx,y)\sim \D_2}[\phi^\prime(y \langle\bv, \bx\rangle)(y\cdot \bx_2)]=0$. 
Substituting for this particular distribution $\D_2$ with noise level $0$ on both points, 
we get: 
\[ p\cdot \phi^\prime(0)\cdot (2\gamma) +(1-p)\phi^\prime(1/2)\cdot(-r)=0 \;. \]
Since $\phi^\prime$ is monotone, we get: 
\[ p|\phi^\prime(0)|\cdot (2\gamma) =(1-p)|\phi^\prime(1/2)|\cdot r \;. \]
By rearranging, we get:
\[
p=\frac{|\phi^\prime(1/2)|\cdot r}{|\phi^\prime(1/2)|\cdot r + 2\gamma|\phi^\prime(0)|} \;.
\] 
By the definition of Case II and the fact that $\phi$ is decreasing and convex, 
we have that: 
$$|\phi^\prime(1/2)|\geq (\eta/2) |\phi^\prime(-1/2)| \geq  (\eta/2) |\phi^\prime(0)| \;.$$ 
Therefore, we can get misclassification error:
\[
\new{\err_{0-1}^{\D_2} (\sgn(\langle\bv,\bx\rangle))}=p \geq \frac{|\phi^\prime(1/2)|\cdot r}{|\phi^\prime(1/2)|\cdot r + \frac{4\gamma}{\eta}|\phi^\prime(1/2)|}=\frac{1}{1+\frac{4\gamma}{\eta r}} \;.
\]
We note that $r$ must be chosen within the interval $\left[0,\sqrt{3}/2-2\gamma\right]$, 
so that the $\gamma$-margin requirement is satisfied.  

For margin $\frac{\sqrt{3}-1}{4}\gamma\leq \frac{\eta r}{4}$, we get $\new{\err_{0-1}^{\D_2} (\sgn(\langle\bv,\bx\rangle))}> 1/2$, 
and we can achieve error exactly $1/2$ by moving the probability mass $p$ 
from $Q_1(0,-2\gamma)$ to $Q_3(0,-\frac{\eta r}{2})$. 
If $\gamma\geq \frac{\eta r}{4}$, then $\new{\err_{0-1}^{\D_2} (\sgn(\langle\bv,\bx\rangle))}\geq \frac{\eta r}{4\gamma}\geq \frac{\eta r}{8\gamma}$. 
The last inequality comes from the fact that we can pick $r= 1/2 \leq \sqrt{3}/2-2\gamma$.    
This completes the proof of Theorem~\ref{thm:lb-surr}.
\end{proof}

\subsection{Lower Bound Against Convex Surrogate Minimization Plus Thresholding} \label{ssec:thresh}
The lower bound established in the previous subsection does not preclude the possibility
that our algorithmic approach in Section~\ref{sec:alg} giving misclassification error $\approx \eta$
can be improved by replacing the $\LR$ function by a different convex surrogate.
In this section, we prove that using a different convex surrogate in our thresholding approach
indeed does not help.

That is, we show that any approach which attempts 
to obtain an accurate classifier by considering a thresholded region 
cannot get misclassification error better than $\Omega(\eta)$ within that region, i.e., 
the bound of our algorithm cannot be improved with this approach. 
Formally, we prove:

\begin{theorem} \label{thm:lb-thresh}
Consider the family of algorithms that produce a classifier $\sgn(\langle\bw^{\ast}, \bx\rangle)$, 
where $\bw^{\ast}$ is the minimizer of the function 
$G(\bw)=\E_{(\bx,y)\sim \D}[\phi(y\langle \bw,\bx\rangle)]$. 
For any decreasing convex function $\phi: \mathbb{R}\rightarrow \mathbb{R}$, 
there exists a distribution $\D$ over $\B_2 \times \{\pm 1\}$with margin $\gamma\leq \sqrt{3}/8$ 
such that the classifier $\sgn(\langle\bw^{\ast}, \bx\rangle)$ misclassifies a 
$(1-O(\gamma))\cdot\Omega(\eta)$ fraction of the points $\bx$ that lie in the region $\{\bx: \langle \bw, \bx \rangle>T \}$ 
for any threshold $T$.      
\end{theorem}

\begin{proof}
Our proof proceeds along the same lines as the proof of Theorem~\ref{thm:lb-surr}, but with some crucial modifications. In particular, we argue that Case I above remains unchanged but Case II requires a different construction.

Firstly, we note that the points $Q_1,Q_2$ in Case I are the only points 
that are assigned non-zero mass by the distribution and they are at equal 
distance $z$ from the output classifier's hyperplane. Therefore, any set 
of the form $\Ind_{\langle\bv, \bx\rangle >T}$, where $\bv$ is the unit 
vector perpendicular to the hyperplane, will either contain the entire 
probability mass or $0$ mass. Thus, for all the meaningful choices 
of the threshold $T$, we get the same misclassification error as with $T=0$. 
This means that the example distribution and the analysis for Case I remain unchanged. 

However, Case II in the proof of Theorem~\ref{thm:lb-surr} requires modification as the points $Q_1,Q_2$ 
are at different distances from the classifier's hyperplane. 

\new{Here we will restrict our attention to the case where the distances 
of the two points from the classifier's hyperplane are actually equal 
and get a lower bound nearly matching the upper bound in Section \ref{sec:alg}. 
This lower bound applies, due to reasons explained above, 
to all approaches that use a combination of minimizing a convex surrogate function and thresholding.}      

\begin{figure}
\begin{center}
\includegraphics[scale=0.2]{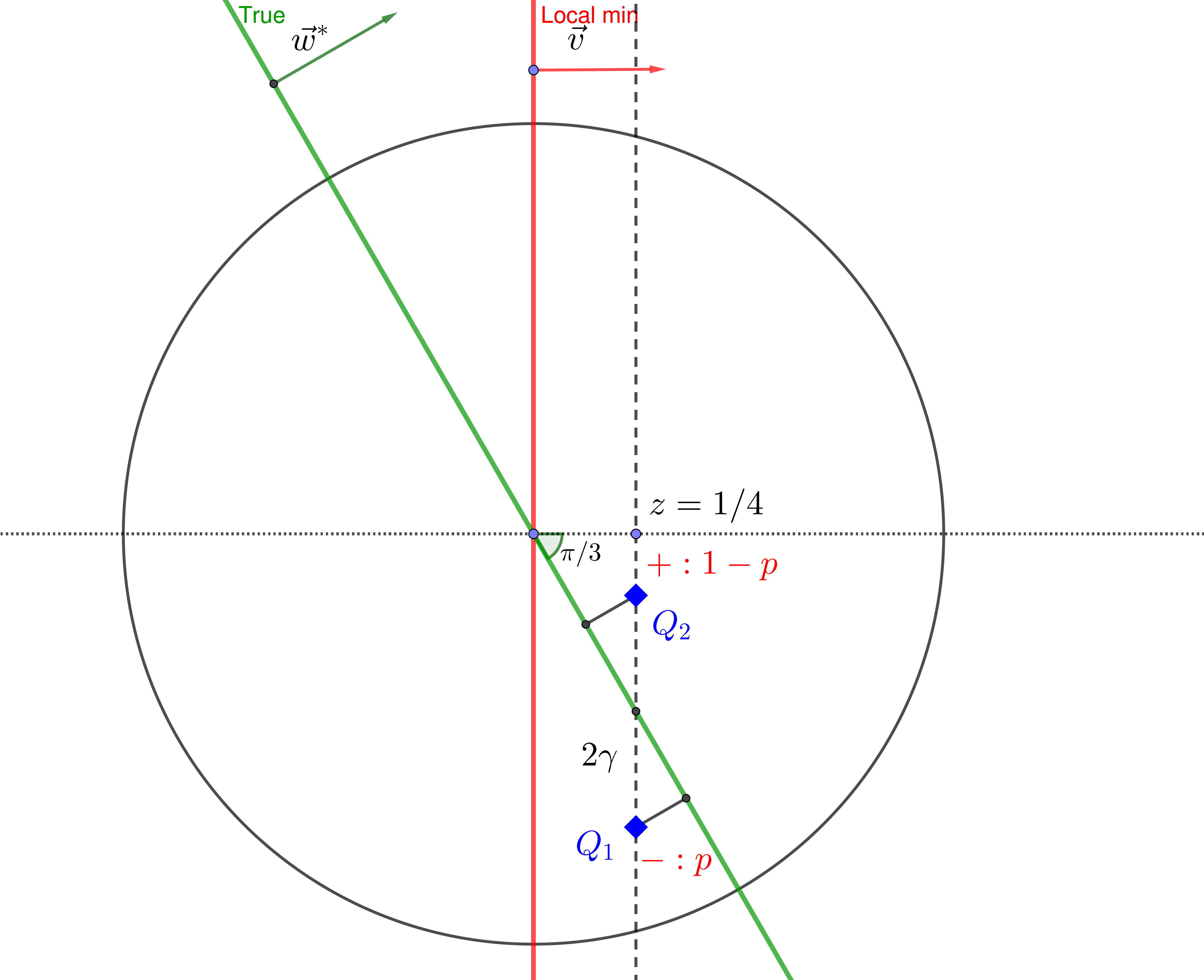}
\end{center}
\caption{Probability Distribution for Modified Case II.}
\label{fg:case2b}
\end{figure}

\paragraph{Modified Case II:}
We recall that in this case the following assumption on the function $\phi$ holds: 
For all $z \in [0,\sqrt{3}/2]$ it holds $|\phi^\prime(z)|\geq\frac12 \frac{\eta}{1-\eta}|\phi^\prime(-z)|$.    

\smallskip

\noindent The new distribution \new{$\D_2^\prime$} is going to be as shown in Figure \ref{fg:case2b}. 
That is, we assign mass $p$ on the point $Q_1(1/4,\sqrt{3}/4+2\gamma)$ 
and mass $1-p$ on the point $Q_2(1/4,\sqrt{3}/4-2\gamma)$. 

Similarly to the previous section, we use the equation:
 $\E_{(\bx,y)\sim \D}[\phi^\prime(y\langle \bv, \bx\rangle)(y\cdot \bx_2)]=0$, 
 that holds for $\bv$ being the  minimum of $G(w)=\E_{(\bx,y)\sim D}[\phi(y\langle \bw, \bx\rangle)]$, 
 to get:
\[
p\cdot \phi^\prime(-1/4)\cdot \left(\sqrt{3}/4+2\gamma\right)+ 
(1-p)\cdot \phi^\prime(1/4)\cdot \left[-\left(\sqrt{3}/4-2\gamma  \right)\right] =0 \;,
\]  
or equivalently:
\begin{align*}
p=\frac{|\phi(1/4)|\left(\sqrt{3}/4-2\gamma  \right)}{|\phi(1/4)|\left(\sqrt{3}/4-2\gamma  \right)
+|\phi(-1/4)|\left(\sqrt{3}/4+2\gamma  \right)}&\geq \frac{\left(\sqrt{3}/4-2\gamma  \right)}{\left(\sqrt{3}/4-2\gamma  \right)
+\frac{2(1-\eta)}{\eta}\left(\sqrt{3}/4+2\gamma  \right)}\\
&\geq \frac{\left(\sqrt{3}/4-2\gamma  \right)}{\left(\sqrt{3}/4+2\gamma  \right)}\cdot \frac{1}{1+\frac{2(1-\eta)}{\eta}}\\
&\geq \left(1-8\gamma\sqrt{3}/3\right)\frac{\eta}{4(1-\eta)} \;.
\end{align*}
This completes the proof of Theorem~\ref{thm:lb-thresh}.
\end{proof}

\section{Conclusions} \label{sec:conc}
The main contribution of this paper is the first non-trivial learning
algorithm for the class of halfspaces (or even disjunctions) 
in the distribution-free PAC model with Massart noise. 
Our algorithm achieves misclassification error $\eta+\eps$ in time 
$\poly(d, 1/\eps)$, where $\eta<1/2$ is an upper bound on the Massart noise rate. 

The most obvious open problem is whether this error guarantee
can be improved to $f(\opt)+\eps$ (for some function $f:\R \to \R$ such that $\lim_{x \to 0} f(x)=0$) 
or, ideally, to $\opt+\eps$. It follows from our lower bound constructions that such an improvement 
would require new algorithmic ideas. It is a plausible conjecture that obtaining better error guarantees 
is computationally intractable. This is left as an interesting open problem for future work.
Another open question is whether there is an efficient {\em proper} learner matching the 
error guarantees of our algorithm. We believe that this is possible, building on the ideas in~\cite{DunaganV04}, 
but we did not pursue this direction. 

More broadly, what other concept classes admit non-trivial algorithms in the Massart noise model?
Can one establish non-trivial reductions between the Massart noise model and the agnostic model?
And are there other natural semi-random input models that allow for efficient PAC learning algorithms 
in the distribution-free setting?

\newpage

\bibliographystyle{alpha}
\bibliography{allrefs}

\newcommand{\etalchar}[1]{$^{#1}$}
\begin{thebibliography}{DKK{\etalchar{+}}19}

\bibitem[ABHU15]{AwasthiBHU15}
P.~Awasthi, M.~F. Balcan, N.~Haghtalab, and R.~Urner.
\newblock Efficient learning of linear separators under bounded noise.
\newblock In {\em Proceedings of The 28th Conference on Learning Theory, {COLT}
  2015}, pages 167--190, 2015.

\bibitem[ABHZ16]{AwasthiBHZ16}
P.~Awasthi, M.~F. Balcan, N.~Haghtalab, and H.~Zhang.
\newblock Learning and 1-bit compressed sensing under asymmetric noise.
\newblock In {\em Proceedings of the 29th Conference on Learning Theory, {COLT}
  2016}, pages 152--192, 2016.

\bibitem[ABL17]{ABL17}
P.~Awasthi, M.~F. Balcan, and P.~M. Long.
\newblock The power of localization for efficiently learning linear separators
  with noise.
\newblock {\em J. {ACM}}, 63(6):50:1--50:27, 2017.

\bibitem[AL88]{AL88}
D.~Angluin and P.~Laird.
\newblock Learning from noisy examples.
\newblock {\em Mach. Learn.}, 2(4):343--370, 1988.

\bibitem[Ber06]{Bernholt}
T.~Bernholt.
\newblock Robust estimators are hard to compute.
\newblock Technical report, University of Dortmund, Germany, 2006.

\bibitem[BFKV96]{BlumFKV96}
A.~Blum, A.~M. Frieze, R.~Kannan, and S.~Vempala.
\newblock A polynomial-time algorithm for learning noisy linear threshold
  functions.
\newblock In {\em 37th Annual Symposium on Foundations of Computer Science,
  {FOCS} '96}, pages 330--338, 1996.

\bibitem[BFKV97]{BFK+:97}
A.~Blum, A.~Frieze, R.~Kannan, and S.~Vempala.
\newblock A polynomial time algorithm for learning noisy linear threshold
  functions.
\newblock {\em Algorithmica}, 22(1/2):35--52, 1997.

\bibitem[Blu03]{Blum03}
A.~Blum.
\newblock Machine learning: My favorite results, directions, and open problems.
\newblock In {\em 44th Symposium on Foundations of Computer Science {(FOCS}
  2003)}, pages 11--14, 2003.

\bibitem[Byl94]{Bylander94}
T.~Bylander.
\newblock Learning linear threshold functions in the presence of classification
  noise.
\newblock In {\em Proceedings of the Seventh Annual {ACM} Conference on
  Computational Learning Theory, {COLT} 1994}, pages 340--347, 1994.

\bibitem[Coh97]{Cohen:97}
E.~Cohen.
\newblock Learning noisy perceptrons by a perceptron in polynomial time.
\newblock In {\em Proceedings of the Thirty-Eighth Symposium on Foundations of
  Computer Science}, pages 514--521, 1997.

\bibitem[Dan16]{Daniely16}
A.~Daniely.
\newblock Complexity theoretic limitations on learning halfspaces.
\newblock In {\em Proceedings of the 48th Annual Symposium on Theory of
  Computing, {STOC} 2016}, pages 105--117, 2016.

\bibitem[DKK{\etalchar{+}}16]{DKKLMS16}
I.~Diakonikolas, G.~Kamath, D.~M. Kane, J.~Li, A.~Moitra, and A.~Stewart.
\newblock Robust estimators in high dimensions without the computational
  intractability.
\newblock In {\em Proceedings of FOCS'16}, pages 655--664, 2016.

\bibitem[DKK{\etalchar{+}}17]{DKK+17}
I.~Diakonikolas, G.~Kamath, D.~M. Kane, J.~Li, A.~Moitra, and A.~Stewart.
\newblock Being robust (in high dimensions) can be practical.
\newblock In {\em Proceedings of the 34th International Conference on Machine
  Learning, {ICML} 2017}, pages 999--1008, 2017.

\bibitem[DKK{\etalchar{+}}18]{DKKLMS18-soda}
I.~Diakonikolas, G.~Kamath, D.~M. Kane, J.~Li, A.~Moitra, and A.~Stewart.
\newblock Robustly learning a gaussian: Getting optimal error, efficiently.
\newblock In {\em Proceedings of the Twenty-Ninth Annual {ACM-SIAM} Symposium
  on Discrete Algorithms, {SODA} 2018}, pages 2683--2702, 2018.

\bibitem[DKK{\etalchar{+}}19]{DKK+19-sever}
I.~Diakonikolas, G.~Kamath, D.~Kane, J.~Li, J.~Steinhardt, and Alistair
  Stewart.
\newblock Sever: {A} robust meta-algorithm for stochastic optimization.
\newblock In {\em Proceedings of the 36th International Conference on Machine
  Learning, {ICML} 2019}, pages 1596--1606, 2019.

\bibitem[DKS18]{DKS18a}
I.~Diakonikolas, D.~M. Kane, and A.~Stewart.
\newblock Learning geometric concepts with nasty noise.
\newblock In {\em Proceedings of the 50th Annual {ACM} {SIGACT} Symposium on
  Theory of Computing, {STOC} 2018}, pages 1061--1073, 2018.

\bibitem[DKS19]{DKS19}
I.~Diakonikolas, W.~Kong, and A.~Stewart.
\newblock Efficient algorithms and lower bounds for robust linear regression.
\newblock In {\em Proceedings of the Thirtieth Annual {ACM-SIAM} Symposium on
  Discrete Algorithms, {SODA} 2019}, pages 2745--2754, 2019.

\bibitem[DKW56]{DKW56}
A.~Dvoretzky, J.~Kiefer, and J.~Wolfowitz.
\newblock Asymptotic minimax character of the sample distribution function and
  of the classical multinomial estimator.
\newblock {\em Ann. Mathematical Statistics}, 27(3):642--669, 1956.

\bibitem[Duc16]{Duchi16}
J.~C. Duchi.
\newblock Introductory lectures on stochastic convex optimization.
\newblock {\em Park City Mathematics Institute, Graduate Summer School
  Lectures}, 2016.

\bibitem[DV04a]{DV:04}
J.~Dunagan and S.~Vempala.
\newblock Optimal outlier removal in high-dimensional spaces.
\newblock {\em J. Computer \& System Sciences}, 68(2):335--373, 2004.

\bibitem[DV04b]{DunaganV04}
J.~Dunagan and S.~Vempala.
\newblock A simple polynomial-time rescaling algorithm for solving linear
  programs.
\newblock In {\em Proceedings of the 36th Annual {ACM} Symposium on Theory of
  Computing}, pages 315--320, 2004.

\bibitem[FGKP06]{FGK+:06short}
V.~Feldman, P.~Gopalan, S.~Khot, and A.~Ponnuswami.
\newblock New results for learning noisy parities and halfspaces.
\newblock In {\em Proc. FOCS}, pages 563--576, 2006.

\bibitem[GR06]{GR:06}
V.~Guruswami and P.~Raghavendra.
\newblock {Hardness of learning halfspaces with noise}.
\newblock In {\em Proc.\ 47th IEEE Symposium on Foundations of Computer Science
  (FOCS)}, pages 543--552. IEEE Computer Society, 2006.

\bibitem[Hau92]{Haussler:92}
D.~Haussler.
\newblock {Decision theoretic generalizations of the PAC model for neural net
  and other learning applications}.
\newblock {\em Information and Computation}, 100:78--150, 1992.

\bibitem[Kea93]{Kearns93}
M.~J. Kearns.
\newblock Efficient noise-tolerant learning from statistical queries.
\newblock In {\em Proceedings of the Twenty-Fifth Annual {ACM} Symposium on
  Theory of Computing}, pages 392--401, 1993.

\bibitem[Kea98]{Kearns:98}
M.~J. Kearns.
\newblock Efficient noise-tolerant learning from statistical queries.
\newblock {\em Journal of the ACM}, 45(6):983--1006, 1998.

\bibitem[KKM18]{KlivansKM18}
A.~R. Klivans, P.~K. Kothari, and R.~Meka.
\newblock Efficient algorithms for outlier-robust regression.
\newblock In {\em Conference On Learning Theory, {COLT} 2018}, pages
  1420--1430, 2018.

\bibitem[KLS09]{KLS09}
A.~Klivans, P.~Long, and R.~Servedio.
\newblock Learning halfspaces with malicious noise.
\newblock To appear in \emph{Proc.\ 17th Internat. Colloq. on Algorithms,
  Languages and Programming (ICALP)}, 2009.

\bibitem[KSS94]{KSS:94}
M.~Kearns, R.~Schapire, and L.~Sellie.
\newblock {Toward Efficient Agnostic Learning}.
\newblock {\em Machine Learning}, 17(2/3):115--141, 1994.

\bibitem[LRV16]{LaiRV16}
K.~A. Lai, A.~B. Rao, and S.~Vempala.
\newblock Agnostic estimation of mean and covariance.
\newblock In {\em Proceedings of FOCS'16}, 2016.

\bibitem[LS10]{LongS10}
P.~M. Long and R.~A. Servedio.
\newblock Random classification noise defeats all convex potential boosters.
\newblock {\em Machine Learning}, 78(3):287--304, 2010.

\bibitem[MN06]{Massart2006}
P.~Massart and E.~Nedelec.
\newblock Risk bounds for statistical learning.
\newblock {\em Ann. Statist.}, 34(5):2326--2366, 10 2006.

\bibitem[MT94]{MT:94}
W.~Maass and G.~Turan.
\newblock How fast can a threshold gate learn?
\newblock In S.~Hanson, G.~Drastal, and R.~Rivest, editors, {\em Computational
  Learning Theory and Natural Learning Systems}, pages 381--414. MIT Press,
  1994.

\bibitem[Ros58]{Rosenblatt:58}
F.~Rosenblatt.
\newblock The {P}erceptron: a probabilistic model for information storage and
  organization in the brain.
\newblock {\em Psychological Review}, 65:386--407, 1958.

\bibitem[RS94]{RivestSloan:94}
R.~Rivest and R.~Sloan.
\newblock A formal model of hierarchical concept learning.
\newblock {\em Information and Computation}, 114(1):88--114, 1994.

\bibitem[Slo88]{Sloan88}
R.~H. Sloan.
\newblock Types of noise in data for concept learning.
\newblock In {\em Proceedings of the First Annual Workshop on Computational
  Learning Theory}, COLT '88, pages 91--96, San Francisco, CA, USA, 1988.
  Morgan Kaufmann Publishers Inc.

\bibitem[Slo92]{Sloan92}
R.~H. Sloan.
\newblock Corrigendum to types of noise in data for concept learning.
\newblock In {\em Proceedings of the Fifth Annual {ACM} Conference on
  Computational Learning Theory, {COLT} 1992}, page 450, 1992.

\bibitem[Slo96]{Sloan96}
R.~H. Sloan.
\newblock {\em Pac Learning, Noise, and Geometry}, pages 21--41.
\newblock Birkh{\"a}user Boston, Boston, MA, 1996.

\bibitem[Val84]{val84}
L.~G. Valiant.
\newblock A theory of the learnable.
\newblock In {\em Proc.\ 16th Annual ACM Symposium on Theory of Computing
  (STOC)}, pages 436--445. ACM Press, 1984.

\bibitem[Vap82]{Vapnik82}
V.~Vapnik.
\newblock {\em Estimation of Dependences Based on Empirical Data: Springer
  Series in Statistics}.
\newblock Springer-Verlag, Berlin, Heidelberg, 1982.

\bibitem[YZ17]{YanZ17}
S.~Yan and C.~Zhang.
\newblock Revisiting perceptron: Efficient and label-optimal learning of
  halfspaces.
\newblock In {\em Advances in Neural Information Processing Systems 30: Annual
  Conference on Neural Information Processing Systems 2017}, pages 1056--1066,
  2017.

\bibitem[ZLC17]{ZhangLC17}
Y.~Zhang, P.~Liang, and M.~Charikar.
\newblock A hitting time analysis of stochastic gradient langevin dynamics.
\newblock In {\em Proceedings of the 30th Conference on Learning Theory, {COLT}
  2017}, pages 1980--2022, 2017.

\end{thebibliography}

\newpage

\appendix

\section{Learning Large-Margin Halfspaces with RCN} \label{sec:rcn}

In this section, we show that the problem of learning $\gamma$-margin halfspaces
in the presence of RCN can be formulated as a convex optimization problem that
can be efficiently solved with any first-order method. 
\new{Prior work by Bylander~\cite{Bylander94} used a variant of the  Perceptron algorithm to learn 
$\gamma$-margin halfspaces with RCN. To the best of our knowledge, 
the result of this section is not explicit in prior work.}


In order to avoid problems that would arise if the distribution $\D$ is degenerate 
(i.e., it assigns non-zero mass on a lower dimensional subspace), 
we introduce Gaussian noise to the points of the distribution. That is, 
we sample points $\bx+\br$, where $\br\sim N( \mathbf{0},c^2 \mathbf{I})$ and 
$c \triangleq \frac \gamma { \sqrt{ 2 \log(2/\gamma \eps) } }$.

In particular, we will show that solving the following convex optimization problem:
\begin{equation} \label{eq:conv}
\begin{aligned}
& \underset{\ltwo{\bw}\leq 1}{\text{minimize}}
&& G_\lambda(\bw)=\E_{(\bx,y)\sim \D}\left[ \E_{\br \sim N( \mathbf{0},c^2 \mathbf{I})}[  \LR_{\lambda}(-y \langle \bw, \bx + \br \rangle)] \right]  \;,\\
\end{aligned}
\end{equation}
for $\lambda \triangleq \eta + \frac{ c \eps } { \sqrt{2 \pi} } \approx \eta$ suffices to solve this learning problem.

Intuitively, the idea here is that by adding the right amount of noise $\br$, 
we make sure that: (a) the probability that the true halfspace misclassifies 
the noisy version of a point $\bx$ is negligible, and (b) if a point is misclassified 
by the current halfspace, then it has, on average, a significant contribution to the 
objective function. Therefore, any solution with sufficiently small value yields a halfspace 
misclassifying a small fraction of points.  

As in Section~\ref{ssec:alg-margin}, we choose the parameter $\lambda$ for the $\LR$ function 
such that $G_\lambda(\bw)$ has a slightly negative minimum.  This is done in order to avoid $\bw=\mathbf{0}$ 
being the minimizer of the function $G_\lambda(\bw)$. The minimizer for the convex region $\ltwo{\bw}\leq 1$ 
will instead lie in the (non-convex) set $\ltwo{\bw} = 1$.

We can solve Problem~\eqref{eq:conv} with a standard first-order method through samples using SGD.
Formally, we show the following:

\begin{theorem}\label{thm:rnc}
Let $\D$ be a distribution over $(d+1)$-dimensional labeled examples obtained by an unknown $\gamma$-margin halfspace 
corrupted with RCN at rate $\eta<1/2$. An application of SGD on $G_\lambda(\bw)$ using 
$\tilde{O}(1/(\eps^2 \gamma^4))$ samples returns, with probability $2/3$, 
a halfspace with misclassification error at most $\eta+\eps$.
\end{theorem}

\noindent The rest of this section is devoted to the proof of Theorem~\ref{thm:rnc}.

We consider the contribution to the objective $G_\lambda$ of a single point $\bx$, denoted by $G_\lambda(\bw, \bx)$.
That is, we define $G_\lambda(\bw, \bx) = \E_{y \sim \D_y(\bx)} [ \E_{\br \sim N(0, c^2 I)}[ \LR_{\lambda}(-y \langle \bw, \bx + \br \rangle)] ]$ and write $G_\lambda(\bw) = \E_{\bx \sim \D_\bx} [ G_\lambda(\bw, \bx) ]$. 

We start with the following claim:

\begin{claim} \label{claim:glambda-equiv}
$G_\lambda(\bw,\bx)$ can be rewritten as:
$$(1-2\eta) \cdot \E_{\br \sim N(0, c^2 I)}\big[ 
 | \langle \bw, \bx + \br \rangle | \Ind_{h_\bw(\bx + \br) \ne h_{\bw^\ast}(\bx)} \big] - (\lambda - \eta)  \cdot \E_{\br \sim N(0, c^2 I)}\big[ | \langle \bw, \bx + \br \rangle |\big].$$
\end{claim}
\noindent The proof of the claim follows similarly to the proof of Claim~\ref{claim:relationship} and is omitted.

Given this decomposition, we move on to show that $G_\lambda(\bw^\ast, \bx)$ is sufficiently negative for any $\bx$ 
and provide a lower bound on $G_\lambda(\bw, \bx)$ for any unit vector $\bw$.

\begin{lemma}\label{lem:rcn-opt-bound}
For any $\bx$ such that $| \langle \bw^\ast, \bx \rangle | \ge \gamma$, 
it holds $$G_\lambda(\bw^\ast, \bx) \le - (\lambda - \eta) \gamma /2 = - \tilde {\Omega}( \gamma^2 \eps ) \;.$$
\end{lemma}
\begin{proof}
For any $\bx$ such that $| \langle \bw^\ast, \bx \rangle | \ge \gamma$, we have that
$$\E_{\br \sim N(0, c^2 I)}\big[ | \langle \bw^\ast, \bx + \br \rangle |\big] \ge 
|\langle \bw^\ast, \bx + \E_{\br \sim N(0, c^2 I)}[\br] \rangle | \ge \gamma.$$
Thus, it suffices to show that:
$$ \E_{\br \sim N(0, c^2 I)}\big[ 
|\langle \bw^{\ast}, \bx + \br \rangle | \Ind_{h_{\bw^{\ast}}(\bx + \br)  \ne h_{\bw^\ast}(\bx)} \big] \le (\lambda - \eta) \gamma/2 \;.$$
We have that 
$$\E_{\br \sim N(0, c^2 I)}\big[| \langle \bw^{\ast}, \bx + \br \rangle | \Ind_{h_{\bw^{\ast}}(\bx + \br) \ne h_{\bw^\ast}(\bx)} \big] 
\le \E_{r \sim N(0, c^2)}\big[r \Ind_{r \ge \gamma} \big] = \frac {c} {\sqrt{2 \pi} } \exp( -(\gamma/c)^2/2) \;.$$
The choice of $c$, implies that 
$\frac {c} {\sqrt{2 \pi} } \exp( -(\gamma/c)^2/2) = \frac {c} {\sqrt{2 \pi} } \eps \gamma/2 = 
(\lambda - \eta)\gamma/2$.
\end{proof}

\begin{lemma}\label{lem:rcn-sgd-bound}
For any unit vectors $\bw,\bx$, it holds 
$$G_\lambda(\bw, \bx) \ge \frac {2 c } {\sqrt{2 \pi} } \left( (1-2\eta) \Ind_{h_\bw(\bx) \neq h_{\bw^\ast}(\bx)} - \eps \right) \;.$$
\end{lemma}
\begin{proof}
To bound the second term in Claim~\ref{claim:glambda-equiv}, we note that for any $\bx, \bw$, we have that
$$\E_{\br \sim N(0, c^2 I)}\big[ | \langle \bw, \bx + \br \rangle |\big] \le 1 + c \le 2 \;.$$
To bound the first term, note that for any $\bx$ such that $h_\bw(\bx) \neq h_{\bw^\ast}(\bx)$, it holds
$$\E_{\br \sim N(0, c^2 I)}\big[| \langle \bw, \bx + \br \rangle | \Ind_{h_\bw(\bx + \br) \neq h_{\bw^\ast}(\bx)} \big] 
\ge  \E_{r \sim N(0, c^2)}\big[ r \Ind_{r \ge 0} \big]  \ge \frac{2 c}{ \sqrt{2 \pi}} \;.$$
Combining the above gives Lemma~\ref{lem:rcn-sgd-bound}.
\end{proof}

\begin{proof}[Proof of Theorem \ref{thm:rnc}]
Taking expectation in Lemma~\ref{lem:rcn-opt-bound}, 
we get that $G_\lambda(\bw^\ast) \le - \tilde {\Omega}( \gamma^2 \eps )$.
From the guarantees of SGD (Lemma~\ref{lm:sgd_guarantees}), running SGD with $\tilde{O}(1/(\eps^2 \gamma^4))$ 
iterations and samples gives a point $\bw$ where $G_\lambda(\bw) \le G_\lambda(\bw^\ast) + O(\eps \gamma^2) \le 0$. 

Furthermore, taking expectation in Lemma~\ref{lem:rcn-sgd-bound}, we obtain that 
\begin{equation}\label{eq:sgd-guarantee}
  (1-2\eta) \pr_{\bx \sim \D_\bx}[h_\bw(\bx) \neq h_{\bw^\ast}(\bx)] \le \eps.
\end{equation}
Overall, the misclassification error of $h_\bw$ is equal to 
\begin{align*}(1-\eta) & \pr_{\bx \sim \D_\bx}[h_\bw(\bx) \neq h_{\bw^\ast}(\bx)] + \eta (1 - \pr_{\bx \sim \D_\bx}[h_\bw(\bx) \neq h_{\bw^\ast}(\bx)]) = \eta + (1-2 \eta) \pr_{\bx \sim \D_\bx}[h_\bw(\bx) \neq h_{\bw^\ast}(\bx)] \;.
\end{align*}
From~\eqref{eq:sgd-guarantee} we obtain that the above is at most $ \eta + \eps$. 
This completes the proof of Theorem \ref{thm:rnc}.
\end{proof}

\end{document}